\documentclass{article} 
\usepackage{iclr2027_conference,times}

\usepackage{amsmath,amsfonts,bm}

\def\eqref#1{equation~\ref{#1}}

\def\1{\bm{1}}

\def\vs{{\bm{s}}}

\DeclareMathAlphabet{\mathsfit}{\encodingdefault}{\sfdefault}{m}{sl}
\SetMathAlphabet{\mathsfit}{bold}{\encodingdefault}{\sfdefault}{bx}{n}

\newcommand{\E}{\mathbb{E}}

\newcommand{\R}{\mathbb{R}}

\usepackage{graphicx}
\usepackage{hyperref}
\usepackage{url}
\usepackage{booktabs}
\usepackage{amsmath,amssymb,amsthm}
\usepackage{bm}

\usepackage{multirow}
\usepackage[table]{xcolor}
\usepackage{adjustbox}
\usepackage{algorithm, algorithmic}
\usepackage{enumitem}
\usepackage{tcolorbox}
\usepackage{tabularx}
\usepackage{array}
\usepackage{wrapfig}
\usepackage{makecell}
\newcolumntype{Y}{>{\raggedright\arraybackslash}X}
\newcommand{\pass}[1]{{\scriptsize\textcolor{gray!70}{#1}}}

\newcommand{\inner}[2]{\langle #1,\, #2\rangle}
\newcommand{\fcos}{\varphi^{\cos}}
\newcommand{\fsin}{\varphi^{\sin}}
\newtheorem{lemma}{Lemma}
\newtheorem{proposition}{Proposition}
\newtheorem{theorem}{Theorem}
\newtheorem{corollary}{Corollary}
\newtheorem{remark}{Remark}

\title{Rationally Linearized MLPs}

\author{ Arijit Sehanobish\thanks{Equal Contribution}\footnotemark[1]$^{*\color{purple}{\textbf{1}}}$\quad Byeongchan Kim\footnotemark[1]$^{*\color{purple}{\textbf{2}}}$\quad \textbf{Avinava Dubey}$^{\color{purple}{\textbf{3}}}$\quad
\textbf{Min-hwan Oh}$^{\color{purple}{\textbf{2}}}$\quad \textbf{Krzysztof Choromanski}\footnotemark[1]$^{*\color{purple}{\textbf{4,5}}}$\vspace{0.3cm}\\ 
$^{\color{purple}{\textbf{1}}}$ Kensho Technologies$\ \ $ $^{\color{purple}{\textbf{2}}}$ Seoul National University $\ \ $ $^{\color{purple}{\textbf{3}}}$Google Research  $\ \ $ $^{\color{purple}{\textbf{4}}}$Google DeepMind  $\ \ $ $^{\color{purple}{\textbf{5}}}$Columbia University
}

\iclrfinalcopy 
\begin{document}

\maketitle

\begin{abstract}
We present \textit{Rationally Linearized Multilayer Perceptrons} (\textbf{RaLiMPs}), a new approach to efficiently linearizing MLPs using low-degree rational approximations of standard activation functions. Inspired by polynomial-based MLP linearization, we show that rational approximations can provide accurate representations at substantially lower degree. RaLiMPs exploit these approximations to construct kernel representations with narrow bottlenecks, enabling efficient composition across MLP layers. We provide theoretical guarantees for the resulting approximation and evaluate RaLiMPs across applications ranging from implicit 3D models to large language models. Our experiments demonstrate favorable accuracy--efficiency tradeoffs, with RaLiMPs achieving competitive or improved accuracy while substantially reducing inference cost.
\end{abstract}

\section{Introduction and Related Work}
\label{sec: intro_related}

Nonlinear processing modules are at the heart of modern deep learning and are critical to modeling complex phenomena with machine learning (ML) approaches \citep{goodfellow2016deep, Hinton06}. Virtually all modern ML systems are built around one or both of the fundamental building blocks: \textit{multi-layer perceptrons} (MLPs) and the \textit{attention layers}. Both are highly nonlinear functions of their inputs, and their combination yields high-capacity models that can approximate complex functions while remaining amenable to robust training.

However that expressiveness comes at the cost: a non-negligible computational footprint which is quadratic in the sequence of the input length for the attention layers and quadratic in the dimensionality of the input for the MLPs. That poses practical challenges and led to fruitful research on efficient attention layers in Transformers\citep{yi-tay, elouargui, performers, reformer}. With recent improvements to efficient attention mechanism, reducing the computational cost of MLPs has become increasingly important. This is in part because the highly optimized attention computation often makes MLPs the new computational bottlenecks for Transformers, in particular when $d^{2} \gg L$ for the input dimensionality $d$ and sequence length $L$. Furthermore, MLPs are ubiquitous in ML and are present also in modern ML architectures that are not attention-based, e.g. implicit 3D-models such as \textit{Neural Radiance Fields} (NeRFs) \citep{nerf}, Signed Distance Fields (SDFs) \citep{isdf} or in Physics Informed Neural Networks (PINNs)~\citep{pinn}. 

A single layer of the MLP processing unit for the input vector $\mathbf{x} \in \mathbb{R}^{d}$, can be seen as a map:
\begin{equation}
\label{eq:mlp-single-layer}
\mathbf{x} \rightarrow f(\mathbf{W}\mathbf{x}) \in \mathbb{R}^{d_{\textrm{out}}},    
\end{equation}
where $\mathbf{W} \in \mathbb{R}^{l \times d_{\textrm{out}}}$ is a learned matrix of weights and $f:\mathbb{R} \rightarrow \mathbb{R}$ is a fixed non-linear \textit{activation function} applied element-wise. 
The mapping defined in Eq. \ref{eq:mlp-single-layer} is a non-linear function of $\mathbf{x}$ but it can be approximated by a product of two feature maps,
$f(\mathbf{W}\mathbf{x}) \approx \Psi(\mathbf{W})\Phi(\mathbf{x})$
where $\Phi(\mathbf{x})$ depends only on the input and $\Psi(\mathbf{W})$ depends only on the weights. More specifically, we consider layers of the form
\begin{equation}
\label{eq:linearized-mlp-base}
\mathbf{x} \rightarrow \Psi(\mathbf{W})\Phi(\mathbf{x}),
\end{equation}
where $\Phi,\Psi:\mathbb{R}^{d}\rightarrow\mathbb{R}^{F}$ are nonlinear, potentially randomized feature maps, with $\Psi$ applied row-wise to $\mathbf{W}$. We will refer to the map defined above as the $(\Psi,\Phi)$-linearized PL (perceptron layer).
It was shown in \citep{snnk} that for large classes of the activation functions $f$, randomized maps $\Phi=\Phi_{f}, \Psi=\Psi_{f}$ can be constructed with the use of \textit{Random Fourier Features} (RFFs) \citep{rffs, random-kitchen-sinks} in such a way that the mechanism proposed in Eq. \ref{eq:linearized-mlp-base} provides an \textbf{unbiased} estimation of the map from Eq. \ref{eq:mlp-single-layer}:
\begin{equation}
\label{e:unbiasedness}
f(\mathbf{Wx}) = \mathbb{E}[\Psi_{f}(\mathbf{W})\Phi_{f}(\mathbf{x})],    
\end{equation}
with the expectation taken with respect to the randomness in $\Psi_{f}, \Phi_{f}$ (often one can take: $\Psi_{f}=\Phi_{f}$). 
Note that Eq. \ref{eq:linearized-mlp-base} is linear in $\Phi(\mathbf{x})$ and that fact itself directly leads to straightforward computational improvements \citep{snnk}.
The key observation is that even though the map $\mathbf{x} \rightarrow \Psi(\mathbf{W})\Phi(\mathbf{x})$ is still non-linear in $\mathbf{x}$, as a linear transformation of $\Phi(\mathbf{x})$, it \textbf{disentangles} the input $\mathbf{x}$ from the weight matrix $\mathbf{W}$, only to entangle them at the end of the entire computational process via simple dot-product calculations.  

However achieving an approximation that is not only unbiased, but is also highly accurate, i.e. characterized by low variance, is in general a difficult problem. This challenge directly translates to the accuracy gap that is observed between multi-layer modules using $(\Psi,\Phi)$-linearized PLs and regular MLPs. For particular classes of activation functions, such as polynomials, random features based on tensor products can replace RFFs to provide a simple linearized mechanism with low variance \citep{eugens}. 


\begin{wrapfigure}{r}{0.5\textwidth}
  \centering
  \includegraphics[width=0.48\textwidth]{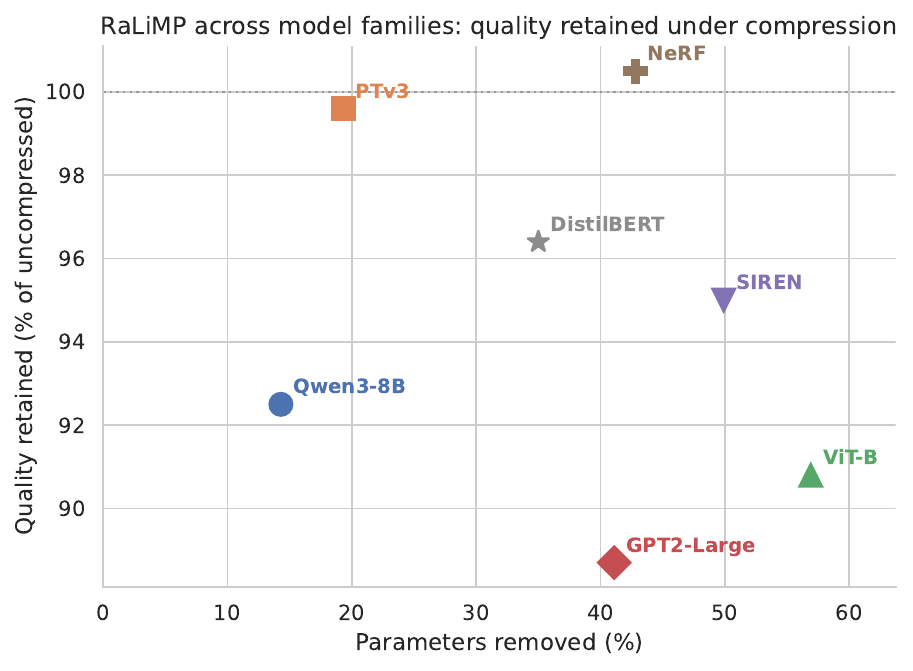}
  \caption{\small{\textbf{Compression–quality trade-off across diverse model families.} Each point represents a model family compressed using RaLiMPs.}}
  \label{fig:compression_vs_accuracy}
\end{wrapfigure}

Several other approaches were proposed to improve the computational efficiency of MLPs, yet most of them departure from the low-variance estimation of the regular MLPs. Many of these approaches focus on replacing $\mathbf{W}$ with its structured counterpart \citep{Dao-2, sindhwani-1, low-rank-swish}, supporting sub-quadratic matrix-vector multiplication.
For instance, Monarch-matrices \citep{monarch} are defined as products of block-diagonal and permutation matrices. Their successors, the so-called \textit{Block Tensor-Train} (BTT) matrices \citep{btt} (containing Monarch-matrices as special cases), derive from the rich theory on Tensorized Neural Networks \citep{tnns}, leveraging sums over Kronecker products.

In this paper, we revisit the $(\Psi,\Phi)$-linearized MLP framework in the context of recent work suggesting that rational functions are particularly effective choices for MLP activations \citep{rational-analysis}. We show that $(\Psi,\Phi)$-linearized MLPs can achieve high accuracy while maintaining an attractive computational footprint when approximating low-degree rational activation functions. To the best of our knowledge, this combination has not been explored previously. Low-degree rational functions are also substantially easier to approximate accurately than the higher-degree functions considered in much of the prior work, making them well suited to random feature methods.
Our contributions are as follows :
\begin{itemize}[nosep]
    \item We present \textit{Rationally Linearized  Multilayer Perceptrons} (RaLiMP), a novel $(\Psi, \Phi)$ linearized layer and prove results on unbiasedness, variance and approximation errors (Sec.~\ref{sec:algorithm}).
    \item We comprehensively test RaLiMP across various tasks like 3D modeling, vision and language modeling against other linearized MLPs and the baselines (Sec.~\ref{sec:exp}).
\end{itemize}

\section{Rationally Linearized Multilayer Perceptrons (RaLiMPs)}
\label{sec:algorithm}

\subsection{Preliminaries}
\label{sec:preliminaries}

Let $\bm{x}\in\R^{d}$ be an input and $\bm{W}\in\R^{l \times d}$ a (pretrained) weight matrix with rows
$\bm{w}_{i}$, and let $f$ be an activation. So the standard feed-forward layer computes $f(\bm{x}\!\cdot\!\bm{w}_i)$. We do not
linearize $f$ directly: we replace it by a rational surrogate $R\approx f$ and linearize $R$. In particular, we construct feature maps $\Phi:\R^{d}\!\to\!\R^{F}$ (on the input) and
$\Psi:\R^{d}\!\to\!\R^{F}$ (on each weight row), where
$\bm{a} \cdot \bm{b} \overset{\mathrm{def}}{=} \bm{a}^{\top}\bm{b}$:
\begin{equation}
  \inner{\Phi(\bm{x})}{\Psi(\bm{w})}\;\approx\;R(\bm{x}\cdot\bm{w}),
  \qquad R(t)=P(t)+\sum_{j} \frac{A_j+B_j(t-p_j)}{(t-p_j)^2+q_j^2},
  \qquad t=\bm{x}\cdot \bm{w}
  \label{eq:pf}
\end{equation}
with a partial-fraction decomposition into a polynomial part $P(t)=\sum_{n=0}^{K}c_n t^n$ and
irreducible-quadratic terms (denominators with no real zero). This decomposition is the governing principle of the construction. We build one separable feature for each part:
\begin{itemize}[nosep]
  \item the \textbf{polynomial part} $P$ is handled by product features that are separable and unbiased
  for each monomial $(\bm{x}\cdot\bm{w})^n$ (Sec.~\ref{sec:alg-polynomial});
  \item each \textbf{irreducible-quadratic atom} is handled by writing it as a Laplace integral over
  frequencies $s$ of $\cos(st)$, and estimating each $\cos(st)=\mathrm{Re}\,(e^{ist})$ with separable
  Gaussian features (Secs.~\ref{sec:alg-atoms-base} and~\ref{sec:gen}).
\end{itemize}
A concatenation of separable blocks is separable, so their sum is a single map of the form
Eq.~\ref{eq:pf}. We first present the centered atom $1/(1+t^2)$ (Sec.~\ref{sec:alg-atoms-base}) and give
the general shifted term in Sec.~\ref{sec:gen}. Because $\Phi$ depends only on $\bm{x}$ and $\Psi$ only on
$\bm{W}$, the full output is a matrix-vector product $\bm{y}=\Psi(\bm{W})\Phi(\bm{x})$, where
$\Psi(\bm{W})\in\R^{l\times F}$ stacks $\Psi(\bm{w}_i)$ as rows.

\subsection{Handling Polynomial Parts}
\label{sec:alg-polynomial}
Eq.~\ref{eq:pf} requires only that the polynomial part be realized by a separable block with finite variance. Any such mechanism may be used, and the guarantee of App.~\ref{sec:total_error}
treats it as one additive term. A polynomial linearization is the natural choice. We use the EUGen mechanism \citep{eugens}: for degree $n$ it draws shared Gaussian projections
$\bm G^{(1)},\dots,\bm G^{(n)}\in\R^{m\times d}$ and, since $\E[\bm G_r^{\top}\bm G_r]=\bm I_d$,
estimates $(\bm x\!\cdot\!\bm w)^{n}$ unbiasedly by averaging
$\prod_{i=1}^{n}(\bm G^{(i)}_r\!\cdot\!\bm x)(\bm G^{(i)}_r\!\cdot\!\bm w)$ over rows $r$. This is
separable, with features $\phi_n(\bm u)_r=\prod_{i=1}^{n}(\bm G^{(i)}_r\!\cdot\!\bm u)$ for
$n\ge1$ and $\phi_0(\bm u)=1$, where $\bm u\in\{\bm x,\bm w\}$; the block contributes
$c_0+\sum_{n=1}^{K}c_n\tfrac1m\inner{\phi_n(\bm x)}{\phi_n(\bm w)}\approx P(t)$ and is of width $1+Km$. Note that this polynomial part may also be approximated by any random feature mechanism with more favorable variance or cost guarantees.

\subsection{Handling Irreducible-Quadratic Atoms: centered atom $\frac{1}{1+t^{2}}$}
\label{sec:alg-atoms-base}
This atom is bounded and saturating, whereas any Taylor
truncation diverges for large $t$. We obtain separable features for it in two steps: (1) a deterministic frequency
quadrature, and then (2) separable softmax-kernel features per frequency, whose real and
imaginary parts estimate $\cos(st)$ and $\sin(st)$.

\paragraph{Frequency quadrature.} Using the Laplace identity,
\begin{equation}
  \frac{1}{1+t^2}=\int_{0}^{\infty} e^{-s}\cos(s\,t)\,ds
  \;\approx\;\sum_{k=1}^{n_{\mathrm{freq}}} \gamma_k\,\cos(s_k\,t),
  \label{eq:gl}
\end{equation}
where $\{s_k,\gamma_k\}$ are the Gauss--Laguerre \citep{gauss-laguerre} nodes and weights ($\gamma_k$
carry the $e^{-s}$ factor). Nodes above a cutoff $s_{\max}$ are dropped and $n_{\mathrm{freq}}$
counts those retained, so the sum discards a fraction $\sum_{s_k>s_{\max}}\gamma_k$ of the unit total
weight $\int_0^\infty e^{-s}ds$ in exchange for bounding the feature growth of the next step. Since $e^{-s}$
is a density this integral could instead be sampled; App.~\ref{sec:quadrature_vs_mc} shows why deterministic
nodes are preferable. It remains to write each $\cos(s_k\,t)$ separably.

\subsubsection{The rational random features}
\label{sec:gen_rat}

We now unbiasedly estimate each factor $\cos(s_{k}t)$ (and $\sin(s_{k}t)$). For this we use the softmax-
and Gaussian-kernel approximation mechanism of Performers~\citep{performers} rather than regular random
Fourier features~\citep{rffs}, which linearize shift-invariant kernels and not functions of
$\bm x\!\cdot\!\bm w$. Fix a node $s\equiv s_k$, draw $m$ i.i.d.\ Gaussian
projections $\bm\omega_1,\dots,\bm\omega_m\sim\mathcal N(\bm0,\bm I_d)$, and set the following real scalars for a vector $\bm u$
($=\bm x$ or $\bm w$) and a scaling factor $\tau$:
\begin{equation}
  Z_r(\bm u)=\sqrt{s/2}\,\big(\bm\omega_r\!\cdot\!\bm u/\tau\big),\qquad
  h_r(\bm u)=\tfrac{s}{2}\|\bm u/\tau\|^2,\qquad r=1,\dots,m.
\end{equation}
The rational random feature maps are the two length-$m$ vectors
\begin{equation}
  \boxed{\;
  \fcos(\bm u)_r=\tfrac{1}{\sqrt m}\,e^{Z_r(\bm u)}\cos\!\big(Z_r(\bm u)-h_r(\bm u)\big),
  \qquad
  \fsin(\bm u)_r=\tfrac{1}{\sqrt m}\,e^{Z_r(\bm u)}\sin\!\big(Z_r(\bm u)-h_r(\bm u)\big).\;}
  \label{eq:ratfeat}
\end{equation}
These are the real and imaginary parts of the softmax-kernel feature
$\varphi_\lambda(\bm u)_r=\tfrac1{\sqrt m}\exp\!\big(\lambda\,\bm\omega_r\!\cdot\!\bm u/\tau
-\tfrac12\lambda^2\|\bm u/\tau\|^2\big)$ with the complex scale $\lambda=\sqrt{s/2}\,(1+i)$
($\lambda^2=is$), which satisfies
$\E[\varphi_\lambda(\bm x)_r\varphi_\lambda(\bm w)_r]=\frac{1}{m}\exp(\lambda^2\,\bm x\!\cdot\!\bm w/\tau^2)=e^{ist/\tau^2}$ \citep{performers}.
Since $\tau$ rescales both arguments it acts on the scalar variable as $t\mapsto t/\tau^{2}$. For simplicity, we take $\tau=1$
below and calibrate it in App.~\ref{sec:tau_appendix}.
This partial-fraction, separable-feature view
subsumes prior polynomial and Fourier random-feature constructions of MLPs~\citep{eugens,snnk} as instances
of one kernel representation, which we make precise in App.~\ref{sec:unified_appendix}. Summing the
real (resp.\ imaginary) part of the product over $r$ gives the unbiased estimators of $\cos(s_k t)$ (resp.
$\sin(s_k t)$) as $1/\sqrt m$ factor is already inside the formula in Eq.~\ref{eq:ratfeat}:
\begin{align}
  \widehat C_k(\bm x,\bm w)&=\sum_{r}\big[\fcos(\bm x)_r\fcos(\bm w)_r-\fsin(\bm x)_r\fsin(\bm w)_r\big]
    \;\to\;\cos(s_k t),\label{eq:Cest}\\
  \widehat S_k(\bm x,\bm w)&=\sum_{r}\big[\fcos(\bm x)_r\fsin(\bm w)_r+\fsin(\bm x)_r\fcos(\bm w)_r\big]
    \;\to\;\sin(s_k t).\label{eq:Sest}
\end{align}
Lemma~\ref{lem:unbiased} below verifies these limits. Complete proof is in App.~\ref{sec:unbiased_appendix}.
\begin{lemma}[Unbiasedness]\label{lem:unbiased}
Let $\bm\omega\sim\mathcal N(\bm0,\bm I_d)$ and $t=\bm w\!\cdot\!\bm x$. Then for $\tau=1$,
\begin{equation}
\E\big[\widehat C_k\big]=\cos(s_k t),\qquad
\E\big[\widehat S_k\big]=\sin(s_k t),
\end{equation}
for general $\tau>0$, applying this to $\bm x/\tau$ and $\bm w/\tau$ gives
$\E[\widehat C_k]=\cos(s_k t/\tau^{2})$, $\E[\widehat S_k]=\sin(s_k t/\tau^{2})$.
\end{lemma}
\begin{proof}[Proof sketch]
Let $\tau = 1$. By angle addition, each $\widehat C_k$ summand is $\tfrac1m\,e^{\xi}\cos(\xi-H)$ where $\xi:=\sqrt{s/2}\,\bm\omega\!\cdot\!(\bm x+\bm w)$ a
centered Gaussian with variance $\tfrac s2\|\bm x+\bm w\|^2$ and $H:=\tfrac s2(\|\bm x\|^2+\|\bm w\|^2)$.
Summing them, $\E[\widehat C_k]=\E\,e^{\xi}\cos(\xi-H)$; since
$e^{\xi}\cos\xi=\mathrm{Re}\,e^{(1+i)\xi}$ and $(1+i)^2=2i$, the Gaussian moment generating function
evaluates this to $\cos(st)$ ($\widehat S_k$ case  identical with $\mathrm{Re}\!\to\!\mathrm{Im}$). 
\end{proof}

\paragraph{Feature towers.} We now present the approximation of the centered atom $\frac{1}{1+t^{2}}$ (Eq.~\ref{eq:gl}). In that case, we only need $\widehat C_k$. To write it as an inner product, split its two terms into two sub-blocks and fold the weight
$\gamma_k$ and the minus sign onto the input side. Length-$2m$ towers, per node $k$ are
\begin{equation}
  \phi^{\mathrm{rat}}_k(\bm x)=\big[\ \gamma_k\,\fcos(\bm x)\ ,\ -\gamma_k\,\fsin(\bm x)\ \big],
  \qquad
  \psi^{\mathrm{rat}}_k(\bm w)=\big[\ \fcos(\bm w)\ ,\ \fsin(\bm w)\ \big].
  \label{eq:rattower}
\end{equation}
They satisfy $\phi^{\mathrm{rat}}_k(\bm x) \cdot \psi^{\mathrm{rat}}_k(\bm w)=\gamma_k\widehat C_k\to
\gamma_k\cos(s_k t)$. Concatenating over the $n_{\mathrm{freq}}$ nodes gives
$\phi^{\mathrm{rat}}(\bm x),\psi^{\mathrm{rat}}(\bm w)\in\R^{2 n_{\mathrm{freq}} m}$ with
$\phi^{\mathrm{rat}}(\bm x) \cdot\psi^{\mathrm{rat}}(\bm w)\approx\sum_k\gamma_k\cos(s_k t)\approx
1/(1+t^2)$. Both towers share the same functional form; only the $\bm x$-side carries the coefficients
$\gamma_k$ and signs, so $\psi^{\mathrm{rat}}(\bm w)$ depends on $\bm w$ alone.
The weights $\gamma_k$, initialized at the quadrature values of Eq.~\ref{eq:gl}, and the projections
$\bm\omega_r$ can be made trainable. The general shifted atom (Eq.~\ref{eq:genatom}) additionally uses $\widehat S_k$; its tower is constructed in Sec.~\ref{sec:gen}.

\paragraph{Inference Speedups.} Since each block separates the weights from the inputs, its computation can be fused into a subsequent linear layer. This eliminates intermediate matrices, reducing the deployed model size and accelerating inference~(Sec.~\ref{sec:bundling}).

\subsubsection{Variance Analysis}
\label{sec:variance}
The estimator presented in the previous section is unbiased, but its summand $e^{\xi}\cos(\xi-H)$ is
unbounded. Set $\sigma^2=\tfrac s2\|\bm x+\bm w\|^2/\tau^2$ and let $H=\tfrac s2(\|\bm x\|^2+\|\bm w\|^2)$ as in sketch proof of lemma~\ref{lem:unbiased}, let
$Y^c_r=m\big[\fcos(\bm x)_r\fcos(\bm w)_r-\fsin(\bm x)_r\fsin(\bm w)_r\big]$ and $Y^s_r$ the analogous
$\widehat S_k$ term, so that $\widehat C_k=\tfrac1m\sum_r Y^c_r$ (and likewise for $\widehat S_k$). The
second moments are exact. Detailed proof is in App.~\ref{sec:variance_appendix}.
\begin{proposition}[Second moments]\label{prop:variance}
For each projection, $\E[(Y^c)^2]=\tfrac12\big(e^{2\sigma^2}+\cos(4\sigma^2-2H/\tau^2)\big)$ and
$\E[(Y^s)^2]=\tfrac12\big(e^{2\sigma^2}-\cos(4\sigma^2-2H/\tau^2)\big)$. Consequently, averaging over $m$ i.i.d.\
projections,
\begin{equation}
  \mathrm{Var}\big(\widehat C_k\big)\;\le\;\frac{1}{2m}\Big(e^{2\sigma^2}+1\Big),
  \label{eq:varC}
\end{equation}
\begin{equation}
  \mathrm{Var}\big(\widehat C_k\big)+\mathrm{Var}\big(\widehat S_k\big)
  \;=\;\frac1m\Big(e^{2\sigma^2}-1\Big).
  \label{eq:varsum}
\end{equation}
\end{proposition}
The variance depends on the inputs only through $\|\bm x+\bm w\|$ and grows like $e^{2\sigma^2}$, a norm dependence that persists even though the target $\cos(s_k t)$ is bounded. The estimator is
thus finite-variance whenever the norms are bounded (e.g.\ post-LayerNorm on a sphere). The cutoff $s_{\max}$ and the input scale
$\tau$ are therefore the two quantities that control it. For a
\emph{sum} of atoms the variance is moreover governed by the narrowest pole width
$q_{\min}=\min_j q_j$ (Cor.~\ref{cor:variance_atom}), not by $s_{\max}$ and $\tau$ alone. The scale $\tau$ is
the practical handle, since the exponent enters as $2\sigma^2$; we calibrate it per layer in
App.~\ref{sec:tau_appendix}.



\subsection{Handling Irreducible-Quadratic Atoms: General Shifted Atom}
\label{sec:gen}
The centered atom extends to the general partial-fraction term
\begin{equation}
  g(t)=\frac{A+B(t-p)}{(t-p)^2+q^2},\qquad q>0,
  \label{eq:genatom}
\end{equation}
whose denominator has complex-conjugate roots $p\pm iq$. A rational
activation with real coefficients that is finite on $\R$ has no real pole, so its denominator factors into
such conjugate pairs and $R$ decomposes as the polynomial $P$ plus one atom per simple pair (repeated pairs in App.~\ref{sec:repeated_poles}). Both parts share the centered atom's Laplace representation, at frequency $s/q$ with
the shift $p$ as a phase:
\begin{equation}
  \frac{1}{(t-p)^2+q^2}=\frac1{q^2}\!\int_0^\infty\!\! e^{-s}\cos\!\big(\tfrac{s}{q}(t-p)\big)ds,
  \quad
  \frac{t-p}{(t-p)^2+q^2}=\frac1{q}\!\int_0^\infty\!\! e^{-s}\sin\!\big(\tfrac{s}{q}(t-p)\big)ds.
\end{equation}
Discretizing with the same nodes at rescaled frequency $s_k/q$, and using
$\cos(\tfrac{s_k}{q}(t-p))=\cos(\tfrac{s_k}{q}p)\,\widehat C_k+\sin(\tfrac{s_k}{q}p)\,\widehat S_k$ (with analogous identity for $\sin$), both parts are separable via the same estimators (Eqs.~\ref{eq:Cest}--\ref{eq:Sest}),
evaluated at $s_k\!\to\!s_k/q$. The shift $p$ mixes the two estimators: the even part ($A$) and the odd
part ($B$) each combine $\widehat C_k$ and $\widehat S_k$ with the phase weights above, so the sine
estimator that the centered atom leaves unused is needed here. The centered atom is the case
$p=0,q=1,A=1,B=0$. A finite sum of Eq.~\ref{eq:genatom} blocks concatenates into one separable inner
product, so with $J$ atoms the total width is $F=1+Km+4Jn_{\mathrm{freq}}m$.

\subsection{Approximation Error Guarantees}
Combining the three blocks, the layer approximates \emph{any} rational-representable activation with a controlled error on a bounded interval. We present the theorem below with detailed proof in App.~\ref{sec:total_error}.
\begin{theorem}[Total error on an interval]\label{thm:total}
Fix $[-T,T]$ containing the pre-activations of interest. Let $\varepsilon_{\mathrm{rat}}=\sup_{|t|\le
T}|R(t)-f(t)|$ be the rational fit error, $\varepsilon_{\mathrm{quad}}=\sup_{|t|\le T}|\widehat g(t)-g(t)|$
the quadrature-truncation error (here $g=\sum_j g_j$ is the non-polynomial part of $R$, each $g_j$ of
the form Eq.~\ref{eq:genatom}, and $\widehat g$ replaces each atom's Laplace integral by the retained nodes
$s_k\le s_{\max}$), and $V_{\mathrm{poly}},V_{\mathrm{rat}}$ the variances of the polynomial and rational
blocks ($V_{\mathrm{rat}}$ bounded by Cor.~\ref{cor:variance_atom}). Then for every $|t|\le T$,
\begin{equation}
  \Big(\E\big[\big(\inner{\Phi(\bm x)}{\Psi(\bm w)}-f(t)\big)^2\big]\Big)^{1/2}
  \;\le\;\varepsilon_{\mathrm{rat}}+\varepsilon_{\mathrm{quad}}+\sqrt{V_{\mathrm{poly}}+V_{\mathrm{rat}}}.
  \label{eq:total}
\end{equation}
\end{theorem}
The three terms are controlled independently: more poles reduce $\varepsilon_{\mathrm{rat}}$, more nodes and a
larger cutoff reduce $\varepsilon_{\mathrm{quad}}$, and more projections reduce the variance but raising $s_{\max}$
inflates $V_{\mathrm{rat}}$ exponentially through Cor.~\ref{cor:variance_atom}, so the scale $\tau$
(App.~\ref{sec:tau_appendix}) keeps that exponent in range. The proof, a refinement of
$\varepsilon_{\mathrm{quad}}$, and a concentration tail under bounded projections (Cor.~\ref{cor:tail}) are in App.~\ref{sec:total_error}.

The polynomial and rational blocks differ sharply in how their
variance grows as we ask for more accuracy. A degree-$n$ monomial has relative variance at least $3^{n}-1$,
even when $\bm x\parallel\bm w$ (Prop.~\ref{prop:polyvar}), so
pushing the polynomial block to higher accuracy quickly becomes costly. The rational block behaves
differently. Adding poles lowers its approximation error, yet its variance does not grow with the number which is set by the input scale alone. We can
therefore sharpen the fit by adding poles without increasing variance. App.~\ref{sec:polycost} makes the cost precise (Thm.~\ref{thm:cost}).

To summarize, we provided a general $(\Phi,\Psi)$-linearized feedforward layer, combining (1) quadrature acting on the Laplace-transform formulation of the centered atoms and
(2) random features for the unbiased estimation of the softmax kernel~\citep{performers}. We call the
resulting module a Rationally Linearized Multilayer Perceptron (RaLiMP). RaLiMPs replace MLPs of deep
neural networks with more computationally efficient modules.  Taking empty sum in Eq.~\ref{eq:pf} recovers a polynomial-only layer, so the atoms lower the bias $\varepsilon_{\mathrm{rat}}$ by construction: at equal polynomial degree they cut it by one to two orders of magnitude on standard activations, and they
stay bounded where a polynomial diverges. In the next section, we evaluate the RaLiMP mechanism empirically.


\begin{figure}
    \centering
\includegraphics[width=\linewidth]{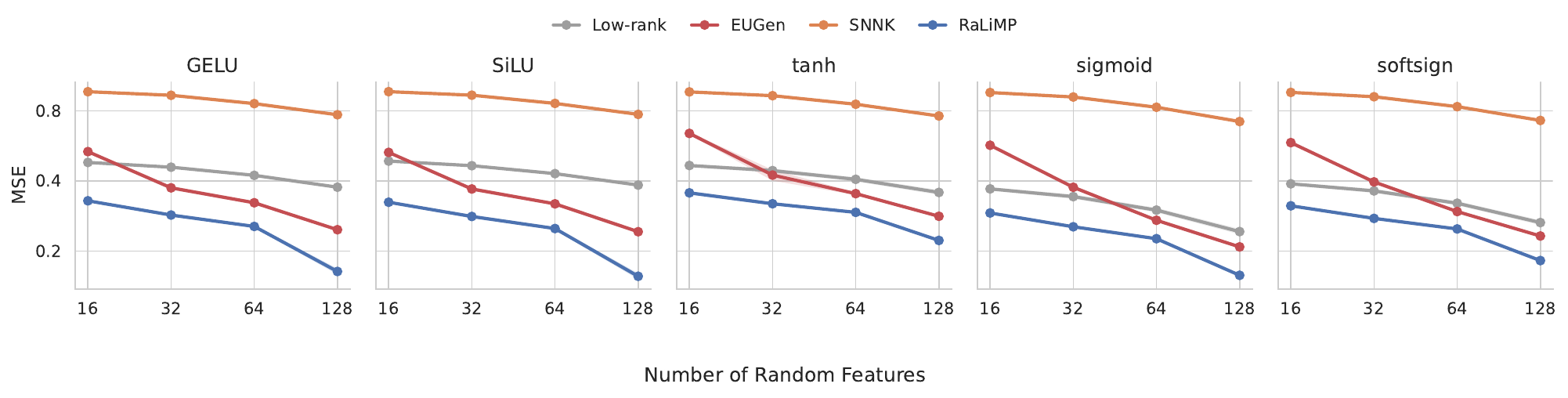}
\vspace{-7mm}
    \caption{\small{\textbf{Synthetic Experiments.}  Approximation for various Feedforward layers with different activations. RaLiMPs beat all other published methods. Error for all methods go down as the number of random features increase. All experiments are run using 10 seeds. }}
    \label{fig:synthetic_expt}
\end{figure}

\section{Experiments}
\label{sec:exp}
In this section, we conduct an exhaustive comparison of RaLIMPs with their counterparts and brute-force baselines. We consider the following research questions : 
\begin{itemize}[nosep]
    \item[(\textbf{RQ1})] How well does RaLiMPs approximate various well-known feed-forward layers?
    \item[(\textbf{RQ2})] How well does RaLiMPs perform in various Transformer architectures across different modalities like text, 2D vision and 3D modeling?
    \item[(\textbf{RQ3})] How well does RaLiMPs generalize beyond Transformer architectures?
\end{itemize}

\subsection{Synthetic Experiments}
In this experiment, we evaluate the approximation quality of RaLiMP layers. Specifically, we test whether RaLiMP layers can approximate the outputs of a fully-connected feedforward layer (FFL), $\mathbf{Y}=f(\mathbf{Wx}+\mathbf{b})$, where $f$ is an activation function. We compare RaLiMP performance with low-rank approximation, SNNK~\citep{snnk} and EUGen~\citep{eugens}.

We report the results in Fig.~\ref{fig:synthetic_expt}, where we compute the MSE loss between the FFL output, $\mathbf{Y}$, and RaLiMP prediction, $\mathbf{\hat Y}$. We observe that RaLiMP achieve significantly lower MSE loss across all activations, showcasing a superior approximation quality of FFLs (additional details in App.~\ref{sec:exp_synth_app}). This provides the answer to (\textbf{RQ1}), showing that RaLiMPs are effective in approximating FFLs.

\subsection{RaLiMPs in Transformers}
\label{sec:ralimp_trans}
Next we focus on replacing feed forward layers in  Transformers for various modalities like 3D modeling, image classification, causal language modeling as well as reasoning (\textbf{RQ2}).
\paragraph{Vision Tasks: Classification \& General Visual Representations}
\label{sec:exp-vit}
We first focus our attention to Vision Transformer (ViT) models \citep{vits}. We start with classic benchmarks involving standard classification datasets: CIFAR-100~\citep{cifar100}, Places365~\citep{places365} and ImageNet-1k~\citep{imagenet}. The results are presented in Fig.~\ref{fig:vit_results_main} (detailed results in Tab.~\ref{table:vit-main-results}). RaLiMP outperforms other efficient MLP substitutes, across different settings (number of layers replaced with the efficient variants). Then, we compare RaLIMP on VTAB-1k~\citep{VTAB}: a diverse and challenging suite of tasks, designed to evaluate general visual representations (fig.~\ref{fig:vit_results_main} with detailed results in Tab. \ref{table-vtab}). RaLiMP provides top performance across all datasets and configurations while reducing memory footprint by up to \textbf{57}\% (Fig.~\ref{fig:efficiency}). VTAB-1k is the hardest setting for these modules. Each task provides only 1000 training examples which are far too few to adapt these layers to new tasks and overcome the approximation errors introduced. This is why weaker baselines collapse while RaLiMP still maintains performance (details in App.~\ref{sec:vit_app}). 
\begin{figure}[h]
    \centering
    \includegraphics[width=0.95\linewidth]{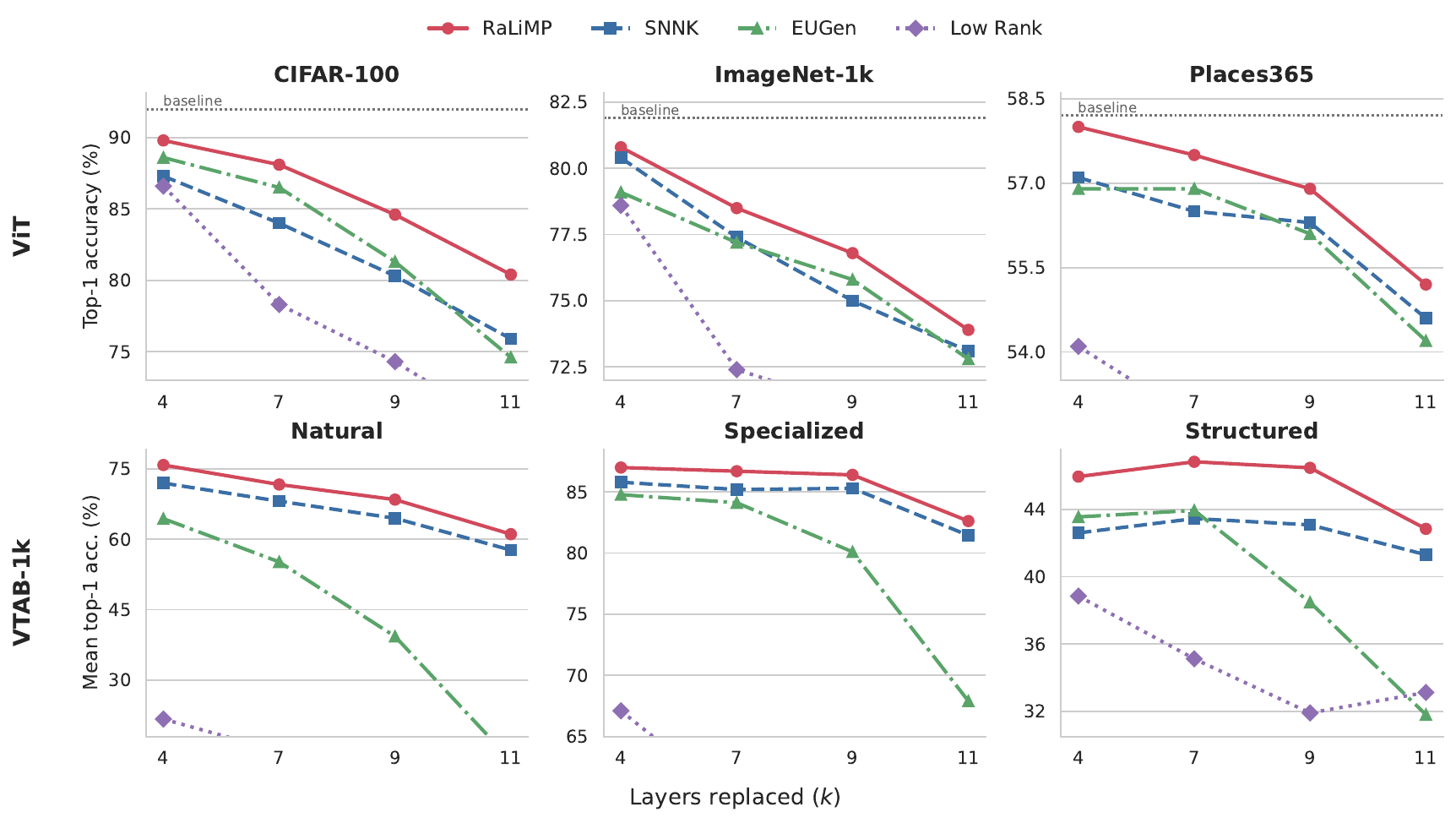}
    \vspace{-4mm}
    \caption{\small{\textbf{ViT Results.} \textit{Top Row :} Classification on CiIFAR-100, ImageNet and Places365. \textit{Bottom Row :} Summarized results on VTAB across different image categories. RaLiMP outperforms all variants. Low Rank and EUGen are consistently weak across all methods at high levels of compression. As more layers are replaced, we get up to \textbf{57}\% reduction in inference parameters. }}
    \label{fig:vit_results_main}
\end{figure}


\begin{table}[h]
\centering
\caption{\small{Qwen3-8B accuracy (\%) at matched inference budget with the top-$k$ feed-forward layers replaced. For the high-variance benchmarks
(AMC23 $n{=}83$, AIME24/25 $n{=}30$, GPQA $ n {=} 198$) each cell reports \textbf{avg@16} with \pass{pass@16} beneath it, at
$t{=}0.6$; the remaining benchmarks are single-sample at a $16$k-token thinking-mode budget, and MMLU is the
57-subject accuracy. Within each compression level the best value is in \textbf{bold}.}}
\label{table:qwen-2}
\renewcommand{\arraystretch}{1.15}
\setlength{\tabcolsep}{6pt}
\resizebox{\textwidth}{!}{%
\begin{tabular}{l ccc cccc c c}
\toprule
& \multicolumn{3}{c}{\textbf{Competition math}\;{\scriptsize(avg@16\,/\,pass@16)}}
& \multicolumn{4}{c}{\textbf{Reasoning}} & & \\
\cmidrule(lr){2-4}\cmidrule(lr){5-8}
Method & AMC23 & AIME24 & AIME25 & GSM8K & MATH500 & Minerva & Olympiad & GPQA & MMLU \\
\midrule
\multicolumn{10}{l}{\emph{$k{=}4$ (last four FFNs replaced)}} \\
\quad SNNK   & \makecell{76.2\\\pass{88.0}} & \makecell{53.8\\\pass{68.8}} & \makecell{47.5\\\pass{70.8}} & 92.6 & 89.9 & 40.4 & 66.3 & \makecell{48.2\\\pass{80.4}} & 71.4 \\
\quad EUGen  & \makecell{79.1\\\pass{90.6}} & \makecell{57.2\\\pass{72.4}} & \makecell{49.4\\\pass{72.0}} & 93.0 & 90.2 & 43.1 & 65.0 & \makecell{49.8\\\pass{80.8}} & 72.0 \\
\quad RaLiMP (ours)
  & \makecell{\textbf{82.2}\\\pass{96.4}} & \makecell{\textbf{61.0}\\\pass{86.0}} & \makecell{\textbf{52.3}\\\pass{77.5}}
  & \textbf{95.5} & \textbf{92.8} & \textbf{45.3} & \textbf{70.4} & \makecell{\textbf{55.6}\\\pass{85.2}} & \textbf{72.7} \\
\midrule
\multicolumn{10}{l}{\emph{$k{=}8$ (last eight FFNs replaced)}} \\
\quad SNNK   & \makecell{66.3\\\pass{84.1}} & \makecell{44.8\\\pass{68.8}} & \makecell{39.5\\\pass{60.8}} & 90.3 & 87.4 & 37.8 & 62.5 & \makecell{42.8\\\pass{77.6}} & 69.6 \\
\quad EUGen  & \makecell{64.1\\\pass{82.4}} & \makecell{46.3\\\pass{70.0}}  & \makecell{41.0\\\pass{62.8}} & 88.4 & 86.2  & 29.8 & 50.5 & \makecell{35.4\\\pass{71.3}} & 65.4 \\
\quad RaLiMP (ours)
  & \makecell{\textbf{78.8}\\\pass{95.2}} & \makecell{\textbf{58.1}\\\pass{79.7}} & \makecell{\textbf{48.6}\\\pass{75.3}}
  & \textbf{95.2} & \textbf{94.5} & \textbf{43.4} & \textbf{69.0} & \makecell{\textbf{51.1}\\\pass{83.9}} & \textbf{71.9} \\
\midrule
Base Qwen3-8B (uncompressed)
  & \makecell{83.0\\\pass{97.6}} & \makecell{67.5\\\pass{90.0}} & \makecell{57.7\\\pass{80.0}}
  & 93.5 & 90.0 & 45.6 & 69.9 & \makecell{59.3\\\pass{86.9}} & 73.0 \\
\bottomrule
\end{tabular}%
}
\end{table}

\paragraph{Causal Language Modeling}
\label{sec:exp-clm}
Our next set of experiments involves large language models (LLMs). We start by pretraining GPT-2 Large (770M) on OpenWebText~\citep{openwebtext} on 22B tokens over 30k iterations. We replace the last $k$ layers with RaLiMP and other linearized MLP methods and measure validation perplexity. Fig.~\ref{fig:clm_results_main} shows that RALiMPs provide accuracy improvements over all other methods (detailed results in Tab.~\ref{table:gpt}). Moreover RaLiMP also improves throughput and shrinks memory footprint by up to \textbf{50}\% (additional training details in App.\ref{sec:clm_app}). 
\begin{figure}[h]
    \centering
    \includegraphics[width=0.95\linewidth]{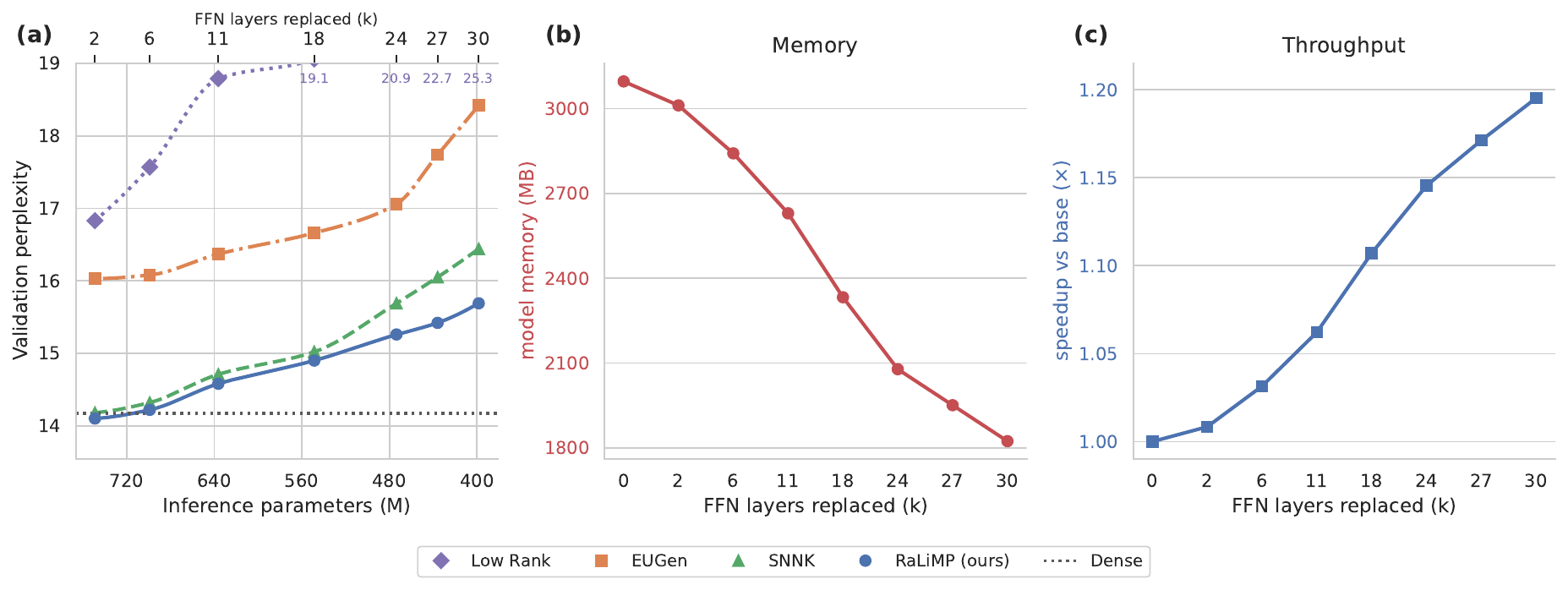}
    \vspace{-4mm}
    \caption{\small{\textbf{Casual Language Modeling with GPT2-L (770M) on OpenWebText.} \textit{Left :} Perplexity (Lower is better) RaLiMP shows the best validation perplexity among all the efficient methods considered with matched inference parameters. \textit{Middle :} Memory usage shrinks to almost \textbf{50}\% as more layers are replaced. \textit{Right :} Throughput increases as well.  }}
    \label{fig:clm_results_main}
\end{figure}

Since Low Rank has been a weak baseline, we drop it in the further expensive experiments.
\paragraph{RaLiMPs in Large Language Models (LLM)} 
We compress Qwen3-8B~\citep{qwen3} by replacing its gated (SwiGLU) feed-forward layers with our rational random feature layer. The rational feature map turns SiLU into an inner product, so the gated product $\mathrm{SiLU}(\mathbf{Wx}) \odot \mathbf{Ux}$ becomes a rank-$r$ Hadamard factorization that never forms the wide intermediate activation (details in App.~\ref{sec:hadamard}) We train the replacement layers at matched inference budget to EUGen and SNNK, and recover the accuracy loss due to compression by on-policy distillation from the uncompressed model (App.~\ref{app:qwen}). At $k{=}8$ the replacement removes $1.17$B parameters ($\sim$14\% of the model) and decoding throughput rises by $6\%$ over the uncompressed model (Tab.~\ref{tab:qwen-efficiency}).  Despite this compression the model stays at near-parity with the base Qwen3-8B across
reasoning and knowledge benchmarks: GSM8K~\citep{cobbe2021gsm8k}, MATH500~\citep{hendrycks2021math}, Minerva~\citep{MINERVA}, MMLU~\citep{mmlu} and OlympiadBench~\citep{OLYMPIADBENCH} are all within roughly a
point with MATH500 and GSM8K slightly exceeding the dense model (Tab.~\ref{table:qwen-2}). On the hardest competition problems (AIME)~\citep{aime_2024, aime_2025} the model's pass@16 remains close to the dense model, indicating the underlying capability is largely preserved. Our detailed evaluation shows that RaLiMP can compress a production-scale transformer while retaining its quality. To the best of our knowledge, this is the first time these methods are used for LLMs. 

Additional finetuning results on GLUE benchmark~\citep{glue} using DistilBERT~\citep{sanh2020distilbertdistilledversionbert} shows that our method can be combined with other compression techniques (Fig.~\ref{fig:glue}). RaLiMPs are competitive with baseline while reducing inference parameters by \textbf{35}\%.

\begin{table}[h]
    \caption{\small{Semantic segmentation performance (mIoU) on indoor (ScanNet, and ScanNet200) and outdoor (nuScenes) benchmarks.
        RaLiMP, SNNK, and EUGen are compared under the same MLP replacement configurations, where $k$ denotes the number of replaced MLP modules in PTv3.
        }}
    \centering
    \begin{adjustbox}{width=\textwidth}
        \begin{tabular}{l|ccc|ccc|ccc}
            \toprule
            \multirow{3}{*}{Methods}
            & \multicolumn{3}{c|}{ScanNet Val}
            & \multicolumn{3}{c|}{ScanNet200 Val}
            & \multicolumn{3}{c}{nuScenes Val} \\
            
            \cmidrule(lr){2-10}
            
            & \multicolumn{9}{c}{MLP modules replaced ($k$)}
            \\

            \cmidrule(lr){2-4}
            \cmidrule(lr){5-7}
            \cmidrule(lr){8-10}
            
            & 5 & 10 & 22
            & 5 & 10 & 22
            & 5 & 10 & 22 \\
            
            \midrule

            EUGen
            & 75.65 & 74.50 & 72.80
            & 34.62 & 33.08 & 30.56
            & 79.14 & 76.34 & 74.84 \\

            SNNK
            & 76.84 & 76.63 & 76.35
            & 35.65 & 34.69 & 34.01
            & 80.76 & 79.92 & 79.79 \\

            RaLiMP (ours)
            & \textbf{78.08} & \textbf{77.54} & \textbf{77.33}
            & \textbf{36.27} & \textbf{35.50} & \textbf{34.90}
            & \textbf{80.81} & \textbf{80.65} & \textbf{80.47} \\

            \midrule
            Baseline (dense, full FT)
            & \multicolumn{3}{c|}{77.57}
            & \multicolumn{3}{c|}{35.26}
            & \multicolumn{3}{c}{80.32} \\
            \bottomrule
        \end{tabular}
    \end{adjustbox}
    \label{table-seg-2}
\end{table}
\begin{table}[h]
    \centering
    \caption{\small{
    Inference efficiency of PTv3 with RaLiMP.
    $k$ denotes the number of replaced MLP modules.
    Memory is measured for the full model, while latency is summed over all MLP modules.
    Measurements use $51{,}200$ points in FP$32$ on an RTX PRO $6000$ Blackwell GPU.
    }}
    \label{tab:ptv3_efficiency}
    \small
    \setlength{\tabcolsep}{5pt}
    \begin{adjustbox}{width=\textwidth}
    \begin{tabular}{lccc}
        \toprule
        Methods
        & Model memory (MB) $\downarrow$
        & Peak GPU memory (MB) $\downarrow$
        & MLP latency (ms) $\downarrow$ \\
        \midrule
        Baseline (dense, full FT) & 184.7 & 849.3 & 12.31 \\
        \midrule
        RaLiMP ($k=5$)           & 168.3 & 832.8 & 9.23  \\
        RaLiMP ($k=10$)          & 151.9 & 817.2 & 6.52  \\
        RaLiMP ($k=22$)          & \textbf{149.1} & \textbf{814.9} & \textbf{6.44}  \\
        \bottomrule
    \end{tabular}
    \end{adjustbox}
\end{table}
\paragraph{Semantic Segmentation}
\label{sec:exp-seg}
We focus on the semantic segmentation tasks on ScanNet~\citep{dai2017scannet}, ScanNet200~\citep{rozenberszki2022language}, and nuScenes~\citep{caesar2020nuscenes}.
Using PTv3~\citep{wu2024point} as the backbone, we compare RaLiMP with SNNK, EUGen, and the fully fine-tuned dense baseline.
The PTv3 backbone contains 22 Transformer blocks, with 14 in the encoder and 8 in the decoder, each containing an MLP module.
We consider three MLP replacement configurations, with $k \in \{5, 10, 22\}$ denoting the number of replaced MLP modules.
We replace the MLP modules in all Transformer blocks for $k=22$, and only in encoder and decoder stages with at least 256 channels for $k=10$.
For $k=5$, we further restrict replacement to the top half of the Transformer layers within each selected stage.
As shown in Table~\ref{table-seg-2}, RaLiMP consistently outperforms the competing efficient methods across all datasets and replacement configurations.
Although segmentation performance generally decreases as more MLP modules are replaced, RaLiMP remains competitive with the dense baseline even at $k=22$, with gaps of only $0.24$ and $0.40$ mIoU points on ScanNet and ScanNet200, respectively, while slightly outperforming it on nuScenes.
With partial replacement, RaLiMP surpasses the baseline on ScanNet200 by \textbf{1.0} mIoU point at $k=5$ and, at $k=10$, nearly matches the baseline on ScanNet while outperforming it on ScanNet200 and nuScenes.

Table~\ref{tab:ptv3_efficiency} quantifies the corresponding inference efficiency gains.
Compared to the dense baseline, RaLiMP reduces memory by up to \textbf{20}\%, with up to \textbf{2}$\times$ faster MLPs.
Notably, $k=10$ configuration captures most of the efficiency gains: replacing the remaining 12 saves only $2.8$ MB of memory and $0.08$ ms of MLP latency, while reducing segmentation accuracy across all benchmarks. This demonstrates the favorable accuracy-efficiency tradeoff of targeting MLPs with at least 256 channels.

\subsection{Beyond Transformers}
\label{sec:exp-nerf}

\begin{figure}[h]
    \centering
\includegraphics[width=0.32\linewidth]{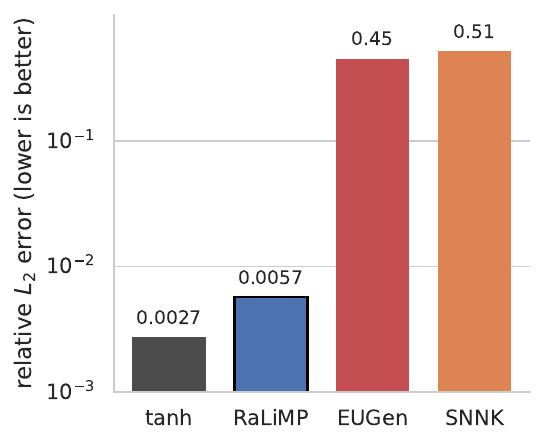}
\includegraphics[width=0.32\linewidth]{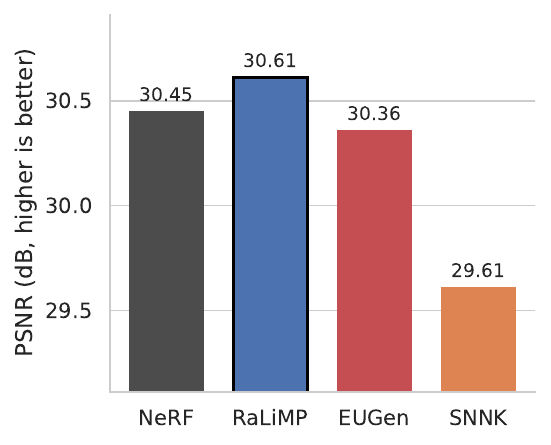}
\includegraphics[width=0.32\linewidth]{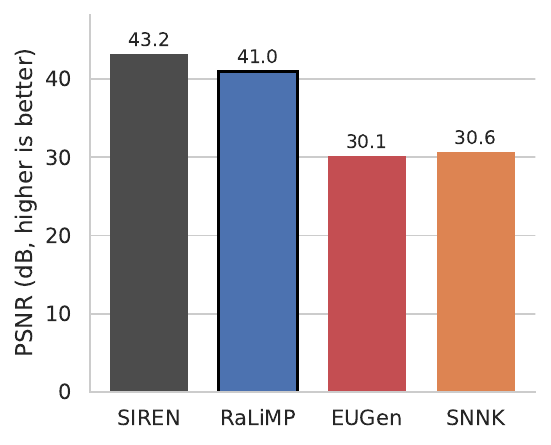}
\vspace{-4mm}
    \caption{\small{\textbf{RaLiMP in PiNNs, NeRF and SIREN.} RaLiMP matches the performance of the baseline PiNN (\textit{Left}) and outperforms NeRF (\textit{Middle}). RaLiMP also matches the SIREN network outperforming the other variants (\textit{Right}). In all cases RaLiMPs reduce model size by \textbf{30}\% -- \textbf{50}\% while maintaining performance. }}
    \label{fig:pinn}
\end{figure}
We start by targeting NeRF \citep{nerf} architectures that rely on MLPs. We replace regular MLPs with their linearized variations: (1) RaLiMPs, (2) SNNKs \citep{snnk} and (3) EUGens \citep{eugens} following the protocol in~\citep{eugens}. RALiMPs is the most robust method among all the considered variants, providing top average score across all 3D scenes (Fig.~\ref{fig:pinn} Middle), while reducing model size by \textbf{30}\% and inference time by \textbf{23}\%. Detailed NeRF results are presented in Table \ref{table:nerf}. Moreover we test RaLiMPs in Siren networks~\citep{siren} where we replaced \textit{all} feed forward layers (Fig.~\ref{fig:siren}). RaLiMP match the performance of the dense model while being \textbf{50}\% smaller (Fig.~\ref{fig:pinn} Right). Finally we apply RaLiMP in Physics Informed Neural Networks (PINNs)~\citep{pinn} for the solution of Burgers equation. RaLiMP close the gap with the dense PINN model while being \textbf{42}\% smaller. Both EUGen and SNNK struggle to solve this equation (Fig.~\ref{fig:pinn} Left and Fig.~\ref{fig:burgers}). For all the experiments, we use the same number of random features for all methods (details in App.~\ref{sec:exp_inr_exp}. These results affirmatively answer \textbf{RQ3}.

\section{Conclusion}

We presented in this paper Rationally Linearized Multilayer Perceptrons (RaLIMPs) - efficient linearized MLPs capable of approximating various commonly used feed-forward layers. We provide a detailed theoretical analysis, complemented by extensive experiments evaluating RaLiMPs across a broad range of tasks, from implicit 3D modeling to visual representation learning, semantic segmentation, and language modeling. Our results confirm the effectiveness of RaLiMP, as an efficient way of linearizing MLPs to improve their computational footprint with minimal quality loss.

\subsection*{AI use statement}

In this work, we used generative AI tools for code assistance and running experiments, generating the plots and for grammar correction. We have not used generative AI tools for idea generation, theoretical results, designing the experiments and the general writing of the paper. We developed the theory and the base code. AI was only used to scale and manage some CLM and Qwen experiments. We have manually vetted and reviewed all the code generated by AI as well as made sure that the plots faithfully and accurate describe our results. We take responsibility for the final content of this work, including text, claims or artifacts produced with the aid of generative AI.

\subsection*{Ethics statement}
This work presents a method for reducing the computational cost of multilayer perceptrons using kernel approximations. It does not involve human subjects, personal data, or automated decision-making. We are not aware of any specific ethical or societal risks arising directly from the proposed method beyond the general potential for more efficient machine learning systems to be used in applications with societal impacts.

\subsection*{Reproducibility statement}

Our mechanism is defined in equation~\ref{eq:ratfeat}. We detail the hardware as well as the hyperparameters for all the experiments in Sec.~\ref{sec:hparam_app}. All datasets used are public datasets and are referenced in the paper. We will open source our code upon acceptance of the manuscript.


%
%

\bibliography{iclr2027_conference}
\bibliographystyle{iclr2027_conference}
\clearpage
\appendix
\section{Theoretical Results}
\label{sec:theory_appendix}
In this section we discuss additional theoretical results.

\subsection{Unbiasedness of the Rational Random Features}
\label{sec:unbiased_appendix}
We present the complete proof of Lemma~\ref{lem:unbiased}, which we restate for convenience of the reader.

\begin{lemma}[Unbiasedness]
Let $\bm\omega\sim\mathcal N(\bm0,\bm I_d)$ and $t=\bm w\!\cdot\!\bm x$. Then for $\tau=1$,
\begin{equation}
\E\big[\widehat C_k\big]=\cos(s_k t),\qquad
\E\big[\widehat S_k\big]=\sin(s_k t),
\end{equation}
and for general $\tau>0$, applying this to $\bm x/\tau$ and $\bm w/\tau$ gives
$\E[\widehat C_k]=\cos(s_k t/\tau^{2})$ and $\E[\widehat S_k]=\sin(s_k t/\tau^{2})$.
\end{lemma}
\begin{proof}
Drop $r$ and write $Z_x=\sqrt{s/2}\,\bm\omega\!\cdot\!\bm x$, $Z_w=\sqrt{s/2}\,\bm\omega\!\cdot\!\bm w$,
$h_x=\tfrac s2\|\bm x\|^2$, $h_w=\tfrac s2\|\bm w\|^2$. By angle addition the $\widehat C_k$ summand
collapses to a single feature,
\begin{equation}
  \fcos(\bm x)\fcos(\bm w)-\fsin(\bm x)\fsin(\bm w)
   = \tfrac1m\,e^{Z_x+Z_w}\cos\!\big((Z_x-h_x)+(Z_w-h_w)\big)=\tfrac1m\,e^{\xi}\cos(\xi-H),
\end{equation}
with $\xi:=\sqrt{s/2}\,\bm\omega\!\cdot\!(\bm x+\bm w)$ and $H:=h_x+h_w$; $\xi$ is centered Gaussian
with $\mathrm{Var}(\xi)=\tfrac s2\|\bm x+\bm w\|^2$. Since $e^{\xi}\cos\xi=\mathrm{Re}\,e^{(1+i)\xi}$
and $(1+i)^2=2i$, the Gaussian moment generating function gives
$\E\,e^{(1+i)\xi}=e^{\frac12(2i)\mathrm{Var}(\xi)}=e^{i\varphi}$ with $\varphi:=\tfrac s2\|\bm x+\bm
w\|^2$, hence $\E[e^{\xi}\cos\xi]=\cos\varphi$ and $\E[e^{\xi}\sin\xi]=\sin\varphi$. Therefore
$\E[e^{\xi}\cos(\xi-H)]=\cos(H-\varphi)$, and
$H-\varphi=\tfrac s2(\|\bm x\|^2+\|\bm w\|^2-\|\bm x+\bm w\|^2)=-st$, giving $\E\,\widehat C_k=\cos(st)$.
The $\widehat S_k$ case is identical with $\mathrm{Re}\!\to\!\mathrm{Im}$. For general $\tau$ the argument
enters as $\bm u/\tau$ on both sides, so $\bm x,\bm w$ are replaced by $\bm x/\tau,\bm w/\tau$ throughout and
the estimators become unbiased for $\cos(st/\tau^{2})$ and $\sin(st/\tau^{2})$.
\end{proof}

\subsection{Variance of the Rational Random Features}
\label{sec:variance_appendix}

We compute the second moments of the estimators of Eqs.~\ref{eq:Cest}--\ref{eq:Sest} exactly, proving
Proposition~\ref{prop:variance}. Throughout, fix a quadrature node $s$,
evaluate the features at $\bm u/\tau$, and set
\begin{equation}
  \sigma^{2}=\tfrac{s}{2}\|\bm x+\bm w\|^{2}/\tau^{2},
  \qquad
  H=\tfrac{s}{2}\big(\|\bm x\|^{2}+\|\bm w\|^{2}\big),
\end{equation}
as in the main text.
Let $Y^{c}_r=m\big[\fcos(\bm x)_r\fcos(\bm w)_r-\fsin(\bm x)_r\fsin(\bm w)_r\big]$ and
$Y^{s}_r=m\big[\fcos(\bm x)_r\fsin(\bm w)_r+\fsin(\bm x)_r\fcos(\bm w)_r\big]$, so that
$\widehat C_k=\tfrac1m\sum_r Y^{c}_r$ and $\widehat S_k=\tfrac1m\sum_r Y^{s}_r$ are averages
of $m$ i.i.d.\ terms.

\medskip
\noindent\textbf{Proposition~\ref{prop:variance}}
\emph{For each projection, $\E[(Y^{c})^{2}]=\tfrac12\big(e^{2\sigma^{2}}+\cos(4\sigma^{2}-2H/\tau^{2})\big)$ and
$\E[(Y^{s})^{2}]=\tfrac12\big(e^{2\sigma^{2}}-\cos(4\sigma^{2}-2H/\tau^{2})\big)$. Consequently, averaging over $m$
i.i.d.\ projections,}
\begin{equation*}
  \mathrm{Var}\big(\widehat C_k\big)\;\le\;\frac{1}{2m}\Big(e^{2\sigma^{2}}+1\Big),
  \qquad
  \mathrm{Var}\big(\widehat C_k\big)+\mathrm{Var}\big(\widehat S_k\big)
  \;=\;\frac1m\Big(e^{2\sigma^{2}}-1\Big).
\end{equation*}

\begin{proof}
Drop the index $r$. By the angle addition used in Lemma~\ref{lem:unbiased},
$Y^{c}=e^{\xi}\cos(\xi-H/\tau^{2})$ and $Y^{s}=e^{\xi}\sin(\xi-H/\tau^{2})$ with
$\xi=\sqrt{s/2}\,\bm\omega\!\cdot\!(\bm x+\bm w)/\tau\sim\mathcal N(0,\sigma^{2})$; that lemma also
gives $\E[Y^{c}]=\cos(st/\tau^{2})$ and $\E[Y^{s}]=\sin(st/\tau^{2})$. For the second moments write
$\cos^{2}\theta=\tfrac12(1+\cos2\theta)$ and $\sin^{2}\theta=\tfrac12(1-\cos2\theta)$, so that
\begin{equation}
  \E\big[(Y^{c})^{2}\big]=\tfrac12\E\big[e^{2\xi}\big]+\tfrac12\E\big[e^{2\xi}\cos(2\xi-2H/\tau^{2})\big].
\end{equation}
The first term is $e^{2\sigma^{2}}$ by the Gaussian moment generating function. For the second, note
$e^{2\xi}\cos(2\xi-2H/\tau^{2})=\mathrm{Re}\,\big[e^{-2iH/\tau^{2}}e^{(2+2i)\xi}\big]$ and $(2+2i)^{2}=8i$, hence
$\E[e^{(2+2i)\xi}]=e^{\frac12(8i)\sigma^{2}}=e^{4i\sigma^{2}}$ and the term equals
$\cos(4\sigma^{2}-2H/\tau^{2})$. The $Y^{s}$ case is identical with the opposite sign. Subtracting
$(\E[Y^{c}])^{2}$ and averaging over $m$ i.i.d.\ projections gives the variance in Eq.~\ref{eq:varC};
the bound follows from $\cos(\cdot)\le1$ and $\cos^{2}(\cdot)\ge0$. For Eqn.~\ref{eq:varsum}, add the two second moments, the cosine terms cancel, leaving $e^{2\sigma^{2}}$ and
$\cos^{2}(\cdot)+\sin^{2}(\cdot)=1$.
\end{proof}

The variance depends on the inputs only through $\|\bm x+\bm w\|$ and grows exponentially in
$2\sigma^2$ and it is symmetric between the cosine and sine channels, so the general shifted atom of
Sec.~\ref{sec:gen} incurs no additional cost per channel.

\begin{corollary}\label{cor:variance_atom}
Draw independent projections per quadrature node, and write $\rho=\|\bm x+\bm w\|$.
\begin{enumerate}
\item[(i)] For the centered atom estimator $\widehat R=\sum_{k}\gamma_k\widehat C_k$,
\begin{equation}
  \mathrm{Var}\big(\widehat R\big)=\sum_{k}\gamma_k^{2}\,\mathrm{Var}\big(\widehat C_k\big)
  \;\le\;\frac{1}{2m}\Big(e^{s_{\max}\rho^{2}/\tau^{2}}+1\Big).
\end{equation}
\item[(ii)] For the general atom of Eq.~\ref{eq:genatom}, whose estimator is
$\widehat g=\sum_k\gamma_k\big(w_{c,k}\widehat C_k+w_{s,k}\widehat S_k\big)$ at frequencies $s_k/q$,
\begin{equation}
  \mathrm{Var}\big(\widehat g\big)\;\le\;\frac{c_A^{2}+c_B^{2}}{m}
  \Big(e^{s_{\max}\rho^{2}/(q\tau^{2})}-1\Big),
  \qquad c_A=\frac{|A|}{q^{2}},\quad c_B=\frac{|B|}{q}.
\end{equation}
\end{enumerate}
\end{corollary}

\begin{proof}
Both use $\gamma_k>0$ and $\sum_k\gamma_k=1$, hence $\sum_k\gamma_k^{2}\le1$, together with independence
across nodes.  For (ii), the per-node coefficients of Sec.~\ref{sec:gen} are the rotation
    $w_{c,k}=c_A\cos\phi_k-c_B\sin\phi_k$, $w_{s,k}=c_A\sin\phi_k+c_B\cos\phi_k$ with
    $\phi_k=s_kp/q$, so $w_{c,k}^{2}+w_{s,k}^{2}=c_A^{2}+c_B^{2}$ for every $k$, and
    Cauchy--Schwarz gives
$\mathrm{Var}(\alpha U+\beta V)\le(\alpha^{2}+\beta^{2})(\mathrm{Var}\,U+\mathrm{Var}\,V)$ for any
$\alpha,\beta$ and any $U,V$; applying this with Eq.~\ref{eq:varsum} at frequency $s_k/q$, for which
$2\sigma_k^{2}=s_k\rho^{2}/(q\tau^{2})$, yields the claim.
\end{proof}

The exponent carries $q$ in the denominator, so for a sum
of atoms the variance is governed by the \emph{narrowest} pole width $q_{\min}=\min_j q_j$, and not by $s_{\max}$ and $\tau$ alone.

\subsection{Total Approximation Error}
\label{sec:total_error}
Three errors accumulate: the rational fit, the truncated quadrature, and the finite projection count. Only the
last is random, and Prop.~\ref{prop:variance} controls it. Since $P$ is unbounded on $\mathbb{R}$, we focus our attention on an interval $[-T,T]$ containing the pre-activations of interest.

Theorem~\ref{thm:total} gives an upper bound to the total $L^{2}$ error. We restate the theorem and prove it below. 

\begin{theorem}[Total error on an interval]
Fix $T>0$ and a target activation $f:[-T,T]\to\R$. Let $R=P+g$ with $g=\sum_j g_j$, each $g_j$ of the form Eq.~\ref{eq:genatom}, be the rational surrogate, its poles expressed in the rescaled variable of App.~\ref{sec:tau_appendix} so that
$f$ and $R$ are compared in the variable the features are evaluated at, and let $\widehat g$ denote the
deterministic value of the truncated quadrature, i.e.\
$g$ with its Laplace integral replaced by the nodes $\{(s_k,\gamma_k):s_k\le s_{\max}\}$. Put
\begin{equation}
  \varepsilon_{\mathrm{rat}}=\sup_{|t|\le T}\big|f(t)-R(t)\big|,
  \qquad
  \varepsilon_{\mathrm{quad}}=\sup_{|t|\le T}\big|g(t)-\widehat g(t)\big|.
\end{equation}
Then for every $\bm x,\bm w$ with $t=\bm w\!\cdot\!\bm x\in[-T,T]$ the RaLiMP estimator is unbiased for the
truncated surrogate,
$\E\big[\inner{\Phi(\bm x)}{\Psi(\bm w)}\big]=P(t)+\widehat g(t)$ (stated for $\tau=1$; general $\tau$
replaces $t$ by $t/\tau^{2}$ inside $\widehat g$, cf.\ App.~\ref{sec:tau_appendix}), and
\begin{equation}
  \Big(\E\big[\big(\inner{\Phi(\bm x)}{\Psi(\bm w)}-f(t)\big)^{2}\big]\Big)^{1/2}
  \;\le\;
  \varepsilon_{\mathrm{rat}}+\varepsilon_{\mathrm{quad}}+\sqrt{V_{\mathrm{poly}}+V_{\mathrm{rat}}},
\end{equation}
where $V_{\mathrm{rat}}$ is bounded by Cor.~\ref{cor:variance_atom}(ii) and $V_{\mathrm{poly}}$ is the variance
of the polynomial block. Furthermore
$\varepsilon_{\mathrm{quad}}\le\varepsilon_{\mathrm{GL}}+(c_A+c_B)\!\!\sum_{s_k>s_{\max}}\!\!\gamma_k$, where
$\varepsilon_{\mathrm{GL}}$ is the Gauss--Laguerre error at full node count; the same bound holds per atom, and
summing over the $J$ atoms gives the truncation error for $g$.
\end{theorem}

\begin{proof}
Unbiasedness is Lemma~\ref{lem:unbiased} per node and channel, plus that of the polynomial block. Decompose
\begin{equation}
  \inner{\Phi(\bm x)}{\Psi(\bm w)}-f(t)
  =\underbrace{\Big(\inner{\Phi(\bm x)}{\Psi(\bm w)}-P(t)-\widehat g(t)\Big)}_{\text{mean zero}}
  +\underbrace{\big(\widehat g(t)-g(t)\big)}_{|\cdot|\le\,\varepsilon_{\mathrm{quad}}}
  +\underbrace{\big(R(t)-f(t)\big)}_{|\cdot|\le\,\varepsilon_{\mathrm{rat}}} .
\end{equation}
The last two terms are deterministic, so the $L^{2}$ triangle inequality gives Eq.~\ref{eq:total}, the first
term's $L^{2}$ norm being the square root of its variance; the two blocks use independent projections, so that
variance is $V_{\mathrm{poly}}+V_{\mathrm{rat}}$. For the last claim, the omitted nodes each contribute at most
$\gamma_k$ times the channel coefficients, as $|\cos|,|\sin|\le1$.
\end{proof}

The trade-off is explicit: more poles cut the first term, more nodes and a larger cutoff the second, more
projections the third, but raising $s_{\max}$ inflates $V_{\mathrm{rat}}$ exponentially through
Cor.~\ref{cor:variance_atom}, so the last two can not be varied independently. The scale $\tau$ of App.~\ref{sec:tau_appendix} keeps that exponent in range.

\paragraph{An exponential tail.} The summands of Prop.~\ref{prop:variance} are lognormal-tailed, with infinite
moment generating function, so Hoeffding- and Bernstein-type bounds need the projections to be bounded.

\begin{corollary}[Concentration under bounded projections]\label{cor:tail}
Fix a node $s\le s_{\max}$ and suppose
$|\bm\omega_r\!\cdot\!(\bm x+\bm w)|\le\beta\|\bm x+\bm w\|$ for all $r=1,\dots,m$. Then with
$B=\exp\!\big(\beta\|\bm x+\bm w\|\sqrt{s_{\max}/2}\,/\tau\big)$, for any $\varepsilon>0$,
\begin{equation}
  \mathbb{P}\Big[\big|\widehat C_k-\cos(s_kt/\tau^{2})\big|\ge\varepsilon\Big]\le2\exp\!\Big(-\frac{m\varepsilon^{2}}{2B^{2}}\Big),
\end{equation}
and likewise for $\widehat S_k$.
\end{corollary}

\begin{proof}
The hypothesis gives $|\xi_r|\le\sqrt{s/2}\,\beta\|\bm x+\bm w\|/\tau$, so $|Y^{c}_r|\le e^{|\xi_r|}\le B$ a.s.
The $Y^{c}_r$ are independent with mean $\cos(s_kt/\tau^{2})$ by Lemma~\ref{lem:unbiased} and
$\widehat C_k=\tfrac1m\sum_rY^{c}_r$, so Hoeffding's inequality applies with range $2B$.
\end{proof}

The hypothesis is not vacuous: $\bm\omega_r\!\cdot\!\bm v/\|\bm v\|$ is standard normal, so a union bound gives
$\beta=\sqrt{2\log(2m/\delta)}$ with probability $\ge1-\delta$, whence
$B=\exp\!\big(\|\bm x+\bm w\|\sqrt{s_{\max}\log(2m/\delta)}/\tau\big)$: the failure probability is explicit,
and the constant is the same exponential in $\|\bm x+\bm w\|/\tau$ that governs the variance.

\subsection{Cost of Accuracy}
\label{sec:polycost}
Sec.~\ref{sec:alg-polynomial} realizes degree-$n$ monomials as products of $n$ independent linear projections. For any such feature map the required width is exponential in $n$. The variance computation rests on the following classical fact about Gaussian moments.

\begin{lemma}[Isserlis' theorem]\label{lem:isserlis}
Let $(X_1,\dots,X_{2n})$ be a centered jointly Gaussian vector. Then
\begin{equation}
  \E\Big[\textstyle\prod_{i=1}^{2n}X_i\Big]=\sum_{\text{pairings}}\ \prod_{\{i,j\}}\E[X_iX_j],
  \label{eq:isserlis}
\end{equation}
the sum running over all partitions of $\{1,\dots,2n\}$ into pairs $\{i,j\}$ (moments of odd order vanish). In
particular, the fourth moment is
\begin{equation}
  \E[X_1X_2X_3X_4]=\E[X_1X_2]\,\E[X_3X_4]+\E[X_1X_3]\,\E[X_2X_4]+\E[X_1X_4]\,\E[X_2X_3].
  \label{eq:isserlis4}
\end{equation}
\end{lemma}

\begin{proposition}[Variance of degree-$n$ product features]\label{prop:polyvar}
Let $\bm G^{(1)},\dots,\bm G^{(n)}\in\R^{m\times d}$ have i.i.d.\ $\mathcal N(0,1)$ entries, independent across
$i$, and put
$X_r=\prod_{i=1}^{n}\big(\bm G^{(i)}_r\!\cdot\!\bm x\big)\big(\bm G^{(i)}_r\!\cdot\!\bm w\big)$, so
$\tfrac1m\sum_rX_r$ estimates $t^{n}$. With $\kappa=t/(\|\bm x\|\|\bm w\|)$,
\begin{equation}
  \mathrm{Var}(X_r)=\big(\|\bm x\|^{2}\|\bm w\|^{2}+2t^{2}\big)^{n}-t^{2n},
  \qquad
  \frac{\mathrm{Var}(X_r)}{t^{2n}}=\Big(\frac{1+2\kappa^{2}}{\kappa^{2}}\Big)^{n}-1\;\ge\;3^{n}-1,
  \label{eq:polyvar}
\end{equation}
with equality iff $\bm x\parallel\bm w$. Hence relative accuracy $\eta$ needs $m\ge(3^{n}-1)/\eta^{2}$.
\end{proposition}

\begin{proof}
Independence across $i$ gives $\E[X_r]=t^{n}$ and
$\E[X_r^{2}]=\prod_i\E[(\bm G^{(i)}_r\!\cdot\!\bm x)^{2}(\bm G^{(i)}_r\!\cdot\!\bm w)^{2}]$. For one Gaussian
row, $(\bm g\!\cdot\!\bm x,\bm g\!\cdot\!\bm w)$ is centered jointly Gaussian with variances
$\|\bm x\|^{2},\|\bm w\|^{2}$ and covariance $t$, so Isserlis' theorem (Lemma~\ref{lem:isserlis}) gives second
moment $\|\bm x\|^{2}\|\bm w\|^{2}+2t^{2}$: the like pairing of the two $\bm g\!\cdot\!\bm x$ and the two
$\bm g\!\cdot\!\bm w$ factors contributes $\|\bm x\|^{2}\|\bm w\|^{2}$, and the two cross pairings contribute
$t^{2}$ each. Substituting $\|\bm x\|^{2}\|\bm w\|^{2}=t^{2}/\kappa^{2}$ gives the
second form; $\kappa^{-2}+2$ is minimized at $\kappa^{2}=1$, where it equals $3$.
\end{proof}

The bound assumes nothing about the geometry and places no bound on the projection entries, which Gaussian rows
would violate. The $3^{n}$ floor is \emph{attained} at $\bm x\parallel\bm w$, so conditioning on the inputs
cannot remove the growth in $n$. 

The contrast with the rational block is: accuracy and cost are controlled by different quantities:
the pole count sets the approximation error, while the variance exponent is set by the input scale alone.

\begin{theorem}[Accuracy and cost are separate factors]\label{thm:cost}
Let $\|\bm x\|=\|\bm w\|$ with $|\kappa|\ge c_0>0$, and let $\tau^{2}=c\,\mathrm{RMS}(\bm w\!\cdot\!\bm x)$
as in App.~\ref{sec:tau_appendix}. Then
\begin{equation}
  \frac{\|\bm x+\bm w\|^{2}}{\tau^{2}}\;\le\;\frac{4}{c\,c_0},
  \label{eq:aligned}
\end{equation}
independent of $d$, so by Theorem.~\ref{thm:total} and Cor.~\ref{cor:variance_atom} an $L^{2}$ error $\varepsilon$
is attained with $m_{\mathrm{rat}}=O\big(\varepsilon^{-2}e^{4s_{\max}/(q_{\min}c\,c_0)}\big)$, with $q_{\min}$
the narrowest pole width (Cor.~\ref{cor:variance_atom}); the exponent does not
depend on $\varepsilon$, on the number of poles, or on $d$. A degree-$K$ product-projection block instead needs,
by Prop.~\ref{prop:polyvar},
\begin{equation}
  m_{\mathrm{poly}}\;\ge\;(3^{K}-1)\,\big(c_K\,t^{K}\big)^{2}\big/\varepsilon^{2},
  \label{eq:polycost}
\end{equation}
with $c_K$ the leading coefficient and $K=K(\varepsilon)$ the degree required for approximation error
$\varepsilon/2$; for non-polynomial $f$, $K(\varepsilon)\to\infty$.
\end{theorem}

\begin{proof}
For Eq.~\ref{eq:aligned}, $\|\bm x\|=\|\bm w\|=r$ gives $\|\bm x+\bm w\|^{2}=2r^{2}(1+\kappa)\le4r^{2}$,
while $|t|=|\kappa|r^{2}\ge c_0r^{2}$ for every such pair, so $\tau^{2}=c\,\mathrm{RMS}(t)\ge c\,c_0r^{2}$.
Eq.~\ref{eq:polycost} follows by discarding all but the degree-$K$ term of the block variance, at least
$c_K^{2}(3^{K}-1)t^{2K}/m$ by Prop.~\ref{prop:polyvar}, and requiring standard deviation at most $\varepsilon$.
\end{proof}

Equation.~\ref{eq:polycost} is only informative if $c_Kt^{K}$ does not vanish as $K$ grows. For $f=\mathrm{silu}$ on
$[-T,T]$ we find $|c_Kt^{K}|=\Theta(1)$ at $|t|=T$ across $K=6,\dots,16$, so the $3^{K}$ survives and the
required width scales as $\varepsilon^{-2-\alpha}$ with $\alpha>0$ growing with $T$.

\begin{remark}
Consider an atom whose poles lie a distance $q$ from $[-T,T]$. Approximating it to error $\varepsilon$ by
polynomials needs degree $K\sim\tfrac{T}{q}\log\tfrac1\varepsilon$, so by Eq.~\ref{eq:polycost} the
polynomial block uses $\varepsilon^{-\Theta(T/q)}$ features, while the rational block reaches the same
$\varepsilon$ with $O_q(\varepsilon^{-2})$. A narrow pole (small $q$) thus enters the rational cost as a
constant and the polynomial cost as a growing exponent in $\varepsilon$.
\end{remark}
\section{Unified Kernel Representation}
\label{sec:unified_appendix}
We relate our construction to prior kernel representations of MLPs~\citep{eugens, snnk}, expanding on
Sec.~\ref{sec:gen_rat}. All three fit the common representation
\begin{equation}
R(t)=\int_{\Xi} K_{\xi}(t)\,d\mu(\xi),\qquad
K_{\xi}(\bm{x}^{\top}\bm{w})=\E_{\omega}\!\left[\phi_{\xi,\omega}(\bm{x})\psi_{\xi,\omega}(\bm{w})\right],
\label{eq:unified}
\end{equation}
with each scalar kernel $K_{\xi}$ carried by separable features, and differ only in how the spectral measure
$\mu$ is obtained. For the polynomial block $\Xi=\mathbb{N}$ and $K_n(t)=t^{n}$, so $\mu$ is the signed
counting measure carrying the coefficients $c_n$, at the cost of Prop.~\ref{prop:polyvar}.

\paragraph{Fourier expansion of a given activation.} Universal random features~\citep{snnk} read $\mu$ off
the Fourier transform of an activation $f$ supplied as input, decomposed into nonnegative parts and sampled
through proposal densities; the resulting $\exp(\bm{\hat x}^{\top}\bm{\hat w})$ is then handled by a
softmax-kernel linearization of the kind used in Sec.~\ref{sec:alg-atoms-base}. We share that device, an
oscillatory integrand against a nonnegative measure, and differ in the direction of the construction: there
$f$ is given and $\mu$ is derived from it, here $\mu$ is the parameterized object and $f$ is whatever it
produces.

\paragraph{Partial fractions.} For an irreducible-quadratic atom the transform is already a single
nonnegative exponential density: by Sec.~\ref{sec:gen}, $\mu$ is $e^{-s}ds$ at frequency $s/q$ for both the
cosine and the sine term, with coefficients $c_A=|A|/q^{2}$ and $c_B=|B|/q$ (Cor.~\ref{cor:variance_atom}),
so no decomposition or reweighting is needed, and Gauss--Laguerre is the quadrature matched to it, its
weight function being exactly $e^{-s}$. The measure is differentiable in the pole parameters, $q$ entering
only by rescaling the nodes to $s_k/q$ and the weights to $\gamma_k/q^{2}$ and $\gamma_k/q$; a gradient
therefore reaches $\mu$ itself, so pole locations and widths are learned rather than read off a fixed
activation. The per-node coefficients are moreover left free, so the block spans
$\{\cos(s_k(t-p)/q),\sin(s_k(t-p)/q)\}$ on the node set, with Eq.~\eqref{eq:pf} as the initialization and
App.~\ref{sec:total_error} covering the family.

\section{Practical Considerations and Extended Functionalities of RaLiMP}
In this section, we discuss additional properties of the RaLiMP layers and discuss some extensions in context of LLMs and guidance on choosing $\tau$ as it controls the variance of the estimator.

\subsection{Choosing the Input Scale $\tau$}
\label{sec:tau_appendix}
The rational features of Eq.~\ref{eq:ratfeat} are evaluated at $\bm u/\tau$ on both sides, so $\tau$ acts on
the scalar variable as $t\mapsto t/\tau^{2}$: the block estimates $\cos(s_k t/\tau^{2})$ and therefore realizes
the \emph{dilated} atom $1/(1+(t/\tau^{2})^{2})$. A single scale thus controls two quantities at once. It sets
the feature exponent $s\|\bm x+\bm w\|^2/\tau^{2}$ of Sec.~\ref{sec:variance}, so larger $\tau$ means smaller
variance; and it sets the effective frequencies $s_k/\tau^{2}$, so larger $\tau$ compresses the Gauss--Laguerre
nodes toward the flat center of the atom and wastes them. The useful regime is $t/\tau^{2}=O(1)$, which places
the nodes on the part of the atom that varies. Concretely we set, \emph{per layer},
\begin{equation}
  \tau^{2} \;=\; c\,\sqrt{\E\big[t^2\big]}\;=\;c\,\mathrm{RMS}(\bm w\!\cdot\!\bm x),
  \label{eq:tau}
\end{equation}
estimated from a small sample of pre-activations, so that $\mathrm{RMS}(t/\tau^{2})=1/c$ and the constant $c$
directly names the typical argument magnitude. Note that the square is what makes this scale-invariant: taking
$\tau$ itself proportional to $\mathrm{RMS}(t)$ would leave $t/\tau^{2}$ \emph{shrinking} as the layer's
pre-activation scale grows.

\subsection{Deterministic Nodes versus Monte Carlo}
\label{sec:quadrature_vs_mc}
The weight $e^{-s}$ in Eq.~\ref{eq:gl} is a probability density, so the frequency integral equals
$\E_{s\sim\mathrm{Exp}(1)}[\cos(st)]$ and could be sampled, as random-feature constructions usually do. Two
reasons for fixed Gauss--Laguerre nodes instead. First, accuracy at equal cost: our per-node coefficients
are trainable, so the fair comparison grants both node sets least-squares-optimal coefficients and differs
only in node placement. Fig.~\ref{fig:spectrum} (right) shows the deterministic nodes ahead at every count,
by $7.4\times10^{-3}$ against $1.2\times10^{-2}$ at $n_{\mathrm{freq}}=3$ and two orders of magnitude at
$n_{\mathrm{freq}}=8$; Monte Carlo nodes cluster near $s=0$, so their near-duplicate features need
fitted coefficients of order $10^{6}$ where ours stay below $1$. Second, the nodes are known before
training, so $s_{\max}$, the discarded weight $\sum_{s_k>s_{\max}}\gamma_k$ and the bound of
Cor.~\ref{cor:variance_atom} are deterministic rather than controlled in expectation, and all randomness in
a RaLiMP sits in the projections $\bm\omega$. 

\begin{figure}
  \centering
  \includegraphics[width=\linewidth]{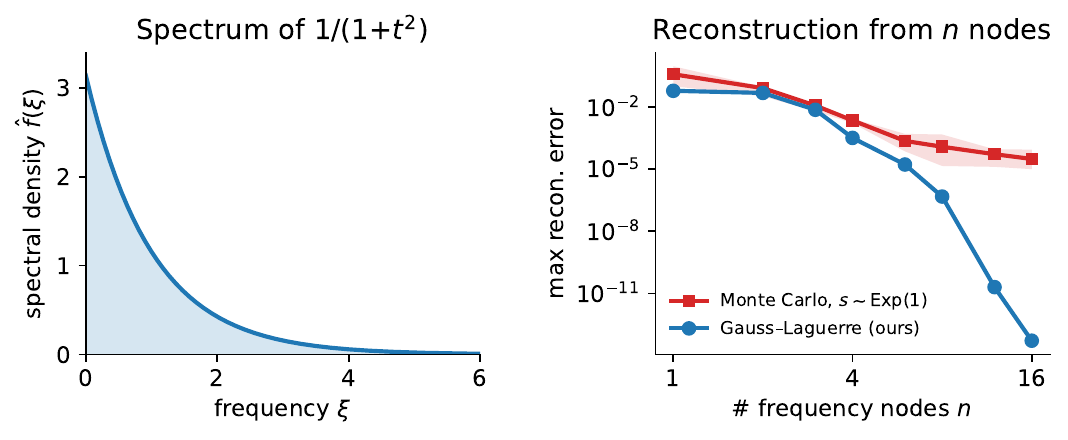}
  \caption{\textbf{Left:} the spectral density $\pi e^{-\xi}$ of $1/(1+t^{2})$ is continuous, so no finite
  set of frequencies is exact and the continuum must be discretized. \textbf{Right:} max-abs reconstruction
  error on $|t|\le1$ from $n$ frequencies, Gauss--Laguerre nodes against $s\sim\mathrm{Exp}(1)$ (mean of 200
  draws, shaded 10--90th percentile), both with least-squares-fitted coefficients.}
  \label{fig:spectrum}
\end{figure}

\subsection{Repeated Conjugate Poles}
\label{sec:repeated_poles}
A conjugate pair of multiplicity $n$ contributes partial-fraction terms with denominators
$((t-p)^2+q^2)^{j}$, $j=1,\dots,n$, which are not of the form of Eq.~\ref{eq:genatom}. These require no new
feature maps: raising the multiplicity only multiplies the Laplace weight by a polynomial in $s$. Indeed
\begin{equation}
  \int_{0}^{\infty} s^{k}e^{-s}e^{ist}\,ds=\frac{k!}{(1-it)^{k+1}},
\end{equation}
whose real and imaginary parts are proper rational functions with denominator $(1+t^{2})^{k+1}$. For $n=2$ the
four functions obtained from $k=0,1$ have numerators $1+t^{2}$, $1-t^{2}$, $t+t^{3}$ and $2t$ over
$(1+t^{2})^{2}$, which span every proper rational function with that denominator; the general case is
identical, and the shifted atom follows by the substitution of Sec.~\ref{sec:gen}. Consequently weights
$\rho_{c}(s)e^{-s}$ and $\rho_{s}(s)e^{-s}$ of degree $n-1$ on the cosine and sine channels suffice, and the
construction is unchanged apart from replacing the quadrature weights by $\gamma_k\rho_{c}(s_k)$ and
$\gamma_k\rho_{s}(s_k)$: the kernel $e^{is_k t}$, the estimators of
Eqs.~\ref{eq:Cest}--\ref{eq:Sest} and the variance bound of Sec.~\ref{sec:variance} all carry over verbatim.

\subsection{Compressing Networks by Combining Layers}\label{sec:bundling}
This is the principal motivation of using kernel methods instead of regular MLP. This was discussed in previous works~\citep{snnk, eugens} but we include it for the convenience of the reader.

In several empirical settings, e.g., training NeRF models and uptraining Transformers, we use an RaLiMP layer followed by a linear layer. In this section, we explain how we can achieve improved efficiency by compressing linear layers with RaLiMPs. Recall our RaLiMP layer is given by the following equation (following the PyTorch convention): 
\begin{equation}~\label{eqn:rf_eqn}
	\mathbf{Y}_1 = \Phi(\mathbf{X})\Psi(\mathbf{W})^{\top}
\end{equation}

If the RaLiMP layer is followed by the linear layer with weight $\mathbf{W}'$ and bias $\mathbf{b}'$, then the results is:
\begin{align*}
	\mathbf{Y}_2 &= \mathbf{Y_1 W'}^{\top} + \mathbf{b}' \\
	&= \Phi(\mathbf{X})\Psi(\mathbf{W})^{\top} \mathbf{W'}^{\top} + \mathbf{b}' \\
	&= \Phi(\mathbf{X_1})\hat{\mathbf W}^{\top} + \mathbf{b}',
\end{align*}
where $\mathbf{\hat{W}} = \mathbf{W}'\Psi(\mathbf{W})$. This simple compression technique allows us to collapse any RaLiMP layer followed by a FFL in our experiments, allowing for more efficient inference. 

\subsection{Efficient Computations with Hadamard-Product Matrices}
\label{sec:hadamard}
Assume that you have a Hadamard-product matrix $\mathbf{M \in \mathbb{R}^{N \times N}}$ defined as:
\begin{equation}
\mathbf{M} = \mathbf{Q} \odot \mathbf{R}    
\end{equation}
for dot-product matrices: $\mathbf{Q} = [\mathbf{a}_{i}^{\top}\mathbf{b}_{j}]_{i=1,...,N}^{j=1,...,N}$ and
$\mathbf{R} = [\mathbf{c}_{i}^{\top}\mathbf{d}_{j}]_{i=1,...,N}^{j=1,...,N}$,
and some $\mathbf{a}_{i},\mathbf{b}_{i},\mathbf{c}_{i},\mathbf{d}_{i} \in \mathbb{R}^{d}$ for $i=1,...,N$. Note that $\mathbf{M}$ can be then re-written as:
\begin{equation}
\label{eq:efficient_hadamard}
\mathbf{M} \overset{\mathbb{E}}{=} \mathbf{V}\mathbf{W}^{T},    
\end{equation}
for matrices $\mathbf{V},\mathbf{W} \in \mathbb{R}^{N \times r}$ with rows defined as follows (where $r \ll N$ is a hyperparameter):
\begin{align}
\begin{split}
\mathbf{V}[i] = (\mathbf{G}^{1}\mathbf{a}_{i}) \odot (\mathbf{G}^{2}\mathbf{c}_{i}) \in \mathbb{R}^{r}, \\
\mathbf{W}[j] = (\mathbf{G}^{1}\mathbf{b}_{j}) \odot (\mathbf{G}^{2}\mathbf{d}_{j}) \in \mathbb{R}^{r},
\end{split}
\end{align}
where $\mathbf{G}^{1},\mathbf{G}^{2} \in \mathbb{R}^{r \times d}$ are independently sampled Gaussian matrices of entries chosen independently from $\mathcal{N}(0,1)$.

Note that Equation \ref{eq:efficient_hadamard} provides an unbiased low-rank factorization of the Hadamard-product matrix $\mathbf{M}$, making efficient matrix-vector multiplications with it feasible. Importantly, it does it by avoiding any computations that would require materializing $d^{2}$-length vectors. 

The above recipe can be trivially extended to non-square matrices.

\paragraph{Implementation of RaLiMP in LLMs}
LLMs use a \emph{gated} (SwiGLU) feed-forward block, $\mathrm{SiLU}(\mathbf{G}x)\odot(\mathbf{U}x)$ followed
by a down-projection $\mathbf{W}_{d}$, which is exactly the Hadamard-product form of
Section~\ref{sec:hadamard} with $\mathbf{Q}=[\mathrm{SiLU}(\mathbf{g}_{j}^{\top}\mathbf{x}_{i})]$ and
$\mathbf{R}=[\mathbf{u}_{j}^{\top}\mathbf{x}_{i}]$ ($\mathbf{g}_{j},\mathbf{u}_{j}$ the rows of
$\mathbf{G},\mathbf{U}$). $\mathbf{R}$ is already a dot-product matrix; the only extra step is to write
$\mathbf{Q}$ as one, which we do by linearizing the scalar SiLU with the rational random-feature map of
Section~\ref{sec:preliminaries}, $\mathrm{SiLU}(\mathbf{g}_{j}^{\top}\mathbf{x}_{i})\approx
\langle \Phi(\mathbf{x}_{i}), \Psi(\mathbf{g}_{j})\rangle$. Section~\ref{sec:hadamard} then applies verbatim
and yields the unbiased rank-$r$ factorization $\mathbf{M}\overset{\mathbb{E}}{=}\mathbf{V}\mathbf{W}^{\top}$.
The weight side $\mathbf{W}_{d}\mathbf{W}$ depends only on $\mathbf{G},\mathbf{U},\mathbf{W}_{d}$ and folds
once, offline, into a single $d\times r$ matrix, so at inference each token costs one feature map plus two
$O(rd)$ matrix products and never materializes the wide intermediate.

\section{Experimental Details}
\label{sec:hparam_app}
In this section we provide experimental details for all the experiments considered in our work. All experiments on ViT, synthetic and Implicit Neural Representations (INR) are run on 1 24 GB A10 GPU. The CLM and the Qwen experiments are run on 1 GPU inside a 8 x 96 GB Blackwell GPU machine with 2048 GB RAM. 
The PTv3 experiments are run on 4 x 84 GB NVIDIA H100 GPUs in a machine with approximately 2048 GB of RAM.
Each subsection states its data and optimizer and then lists the hyperparameters of the replaced layers following notation of Sec.~\ref{sec:gen}. A rational block draws
$n_{\mathrm{freq}}$ Gauss--Laguerre frequencies (Eq.~\ref{eq:gl}) and keeps those with node $s_k\le s_{\max}$; where an experiment lists no $s_{\max}$, all $n_{\mathrm{freq}}$ frequencies are kept. The
nodes spread further from the origin as $n_{\mathrm{freq}}$ grows, so a fixed $s_{\max}$ retains only the lowest-frequency ones: raising $n_{\mathrm{freq}}$ refines the quadrature inside $[0,s_{\max}]$ rather than reaching higher frequencies. If $s_{max}$ is not mentioned then the given experiment keeps all $n_{\mathrm{freq}}$
nodes. On the pretrained backbones the input scale is
$\tau=c\,\mathrm{RMS}(\bm w\!\cdot\!\bm x)$ (Eq.~\ref{eq:tau}) with $c=4$ unless stated otherwise.

\subsection{Synthetic Experiments}
\label{sec:exp_synth_app}
The synthetic experiment benchmark fits one layer
$\bm Y=f(\bm X\bm W^{\!\top})$ with a fixed $\bm W\in\R^{256\times256}$ and $\bm X\in\R^{8192\times256}$, every entry i.i.d.\ $\mathrm{Unif}[0,1]$ and the pre-activation is centered by its mean and both $\bm X$ and
$\bm Y$ are standardized per column, so the reported fit is in standardized units. We take $f\in\{\mathrm{GELU},\mathrm{SiLU},\tanh,\mathrm{sigmoid},\mathrm{softsign}\}$ and train for $10$ seeds using Adam optimizer. The models are trained for $3000$ steps with learning rate $10^{-2}$ and cosine decay along with clipping at $10$. We report the fitting error as train MSE loss. All four methods linear map, the EUGen, SNNK, and ours share the feature budget of 32 random features/rank, RaLimP spends it on the polynomial of degree $K=2$ with $8$ random features and the rest of the $24$ features in the rational block is split among $n_{\mathrm{freq}}=3$ frequencies with $m=4$ features each, one centered summand of Eq.~\ref{eq:pf}, and $\tau=1$.
The companion panel sweeps the budget over various number of random features for our layer.

\subsection{ViT Experiments}
\label{sec:vit_app}
\label{sec:exp_vit}
We use ViT-B/16 pretrained on ImageNet-21k
(\texttt{google/vit-base-patch16-224-in21k}): 12 layers, dimension =$768$, MLP width $3072$ and GELU non-linearity. We
replace the MLP of the last $k$ blocks (i.e.\ blocks $12-k,\dots,11$) and fine-tune the entire network
with cross-entropy, AdamW, a one-cycle schedule with $10\%$ linear warmup, gradient-norm clipping at $1.0$, fp32, and seed $0$. We test across $k = 4, 7, 9, 11$. We run
two regimes: full-dataset fine-tuning on CIFAR-100, Places365 and ImageNet-1k, and VTAB-1k~\citep{VTAB} transfer over its $19$ tasks (each $1000$ training images, the official $800{+}200$ split used at first for hyperparameter tuning). The reported numbers are from the full test set.

The Visual Task Adaptation Benchmark measures how well a pretrained
vision model transfers to new image-classification tasks under a tight data budget. It comprises 19 tasks
grouped into three categories: \emph{Natural} (everyday photographic images, e.g.\ pets, flowers),
\emph{Specialized} (domain images such as medical and satellite imagery), and \emph{Structured} (tasks
requiring geometric reasoning, e.g.\ counting objects or estimating distances). The ``1k'' denotes the
protocol: each task provides only 1{,}000 labeled training examples, with evaluation on the full test set.
This small training budget is what makes the benchmark demanding as there is little data from which to recover the accuracy lost when introducing our replacement layers.

Table~\ref{tab:vit-recipe} gives the
per-dataset settings; for VTAB-1k, Table~\ref{tab:vtab-sweep} lists the values swept around the default
recipe and each method's reported accuracy is its best over the sweep.

\begin{table}[h]
  \centering\small
    \caption{ViT-B/16 fine-tuning settings. VTAB-1k uses random-resized crop only when the source resolution exceeds $100$ pixels.}
  \resizebox{\textwidth}{!}{%
  \begin{tabular}{lccccl}
    \toprule
    Dataset & train / eval & epochs & batch & lr / wd & augmentation \\
    \midrule
    CIFAR-100    & 50k / 10k    & 12--20 & 256 & $5\!\times\!10^{-5}$ / $0.05$ & resize, h-flip \\
    Places365    & 1.71M / 90k  & 15     & 128 & $5\!\times\!10^{-5}$ / $0.05$ & random-resized crop, h-flip \\
    ImageNet-1k  & 1.28M / 50k  & 25--35 & 128--256 & $5\!\times\!10^{-5}$ / $0.05$ & random-resized crop, h-flip \\
    VTAB-1k ($19$ tasks) & 1k / full test & 100 & 64 & $2.5\!\times\!10^{-4}$ / $10^{-3}$ & resize, random-resized crop, h-flip \\
    \bottomrule
  \end{tabular}
  }
  \label{tab:vit-recipe}
\end{table}

\paragraph{Layer hyperparameters.} The random-feature projections are frozen at their random initialization in all experiments. Coefficients are initialized by a ridge-regularized
least-squares fit of GELU on a per-layer range set from the pretrained pre-activations ($2000$ points),
with the quadratic term fixed to EUGen's and the coefficient $A$ fit to the residual. The per-layer scale is
$\tau=4\,\mathrm{RMS}(\bm w\!\cdot\!\bm x)$, measured once from a minibatch of pretrained pre-activations
and then held fixed. This value of $\tau$ for some datasets (in VTAB) can be large so we additionally clamp this value to a fixed number. On VTAB-1k the rational branch is introduced gradually, it is scaled by a factor
that increases linearly from $0$ to $1$ over the first $30\%$ of steps. We use $64$ random features for the larger image classification experiments and $32$ random features for VTAB. RaLiMP splits its budget into $48$, resp. $24$ (for VTAB) features for the polynomial part and the rational block with one centered atom using $16$ resp. $8$ (for VTAB) random features, $n_{\mathrm{freq}}=3$, and $s_{\max}=4$. 

\begin{table}[h]
  \centering\small
    \caption{Hyperparameters swept on VTAB-1k, around the default recipe of Table~\ref{tab:vit-recipe}.
  Deep layers of a pretrained ViT have much larger pre-activations than early ones, so capping $\tau$
  is the one lever specific to our layer.}
  \begin{tabular}{ll}
    \toprule
    Hyperparameter & Values \\
    \midrule
    learning rate & $1.9,\,2.1,\,2.5\ (\times10^{-4})$; up to $10^{-2}$ on the hardest Structured tasks \\
    weight decay & $10^{-3}$, $5\!\times\!10^{-2}$ \\
    scale multiplier $c$ in $\tau$ & $4,\,8,\,10,\,12$ \\
    cap on $\tau$ & none, $10$, $12$, $16$, $20$, $21$ \\
    label smoothing / dropout & $0,\,0.05,\,0.1,\,0.15$ / $0,\,0.1$ \\
    \bottomrule
  \end{tabular}
  \label{tab:vtab-sweep}
\end{table}

\begin{figure}[h]
\centering
\includegraphics[width=0.9\linewidth]{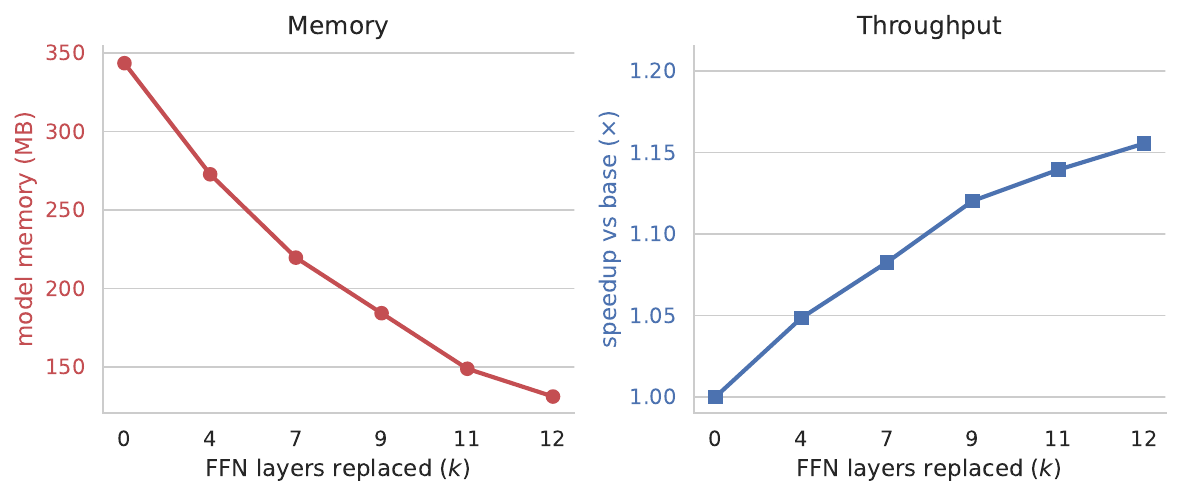}
\caption{\small{RaLiMPs in ViT. Replacing more layers by RaLiMPs reduce memory overhead and speed up inference.}}
\label{fig:efficiency}
\end{figure}

\subsection{Causal Language Modeling}
\label{sec:clm_app}
We use \url{https://github.com/karpathy/nanogpt} implementation of GPT-2. For our experiment, GPT-2 Large ($36$ blocks, $d=1280$, $20$ heads, context $1024$) is trained from scratch on OpenWebText
with the GPT-2 tokenizer and $10$M held-out tokens for validation. We replace the MLP of the last $k$ blocks and run each experiment identically: AdamW ($\beta=(0.9,0.95)$, weight decay $0.1$), lr
$6\!\times\!10^{-4}$ decaying by cosine to $6\!\times\!10^{-5}$ after $500$ warmup iterations, batch
$24\times30$ accumulated steps of $1024$ tokens ($737$k tokens per iteration), bf16, and seed $1337$. Validation perplexity is reported at a fixed budget of $30$k iterations ($\approx\!22$B tokens, generally following the Chinchilla scaling law~\citep{chinchilla}), averaged over $200$ batches. Both EUGen and RaLiMP use degree 2 polynomials. Both SNNK and EUGen uses 128 random features, while RaLiMP uses 128 random features split between the polynomial block with 120 features and rational block with $n_{\mathrm{freq}}=2$ frequencies and $m=4$ projections each with trainable $A,B,p,q$ of Eq.~\ref{eq:genatom} (i.e. $8$ random features), and $\tau=3$.

\subsection{PTv3 Experiments}
We evaluate RaLiMP on three indoor and outdoor 3D semantic segmentation benchmarks:
\begin{itemize}[leftmargin=2.5em]

    \item \textbf{ScanNet v2}~\citep{dai2017scannet}: A large-scale indoor RGB-D dataset with 20 semantic categories. We use the standard split of 1,201 training and 312 validation scans.

    \item \textbf{ScanNet200}~\citep{rozenberszki2022language}: An extension of ScanNet v2 with 200 fine-grained categories, designed to evaluate robustness under long-tailed class distributions.

    \item \textbf{nuScenes}~\citep{caesar2020nuscenes}: A large-scale outdoor autonomous driving dataset with 1,000 urban scenes. We use it to evaluate RaLiMP on sparse LiDAR point clouds across 16 classes.
\end{itemize}
We use the official PTv3 implementation provided by Pointcept~\citep{pointcept2023} and replace selected MLP modules with RaLiMP.
The backbone uses serialized softmax attention with FlashAttention, and the patch sizes of both the encoder and decoder are set to $1024$.
The serialization patterns Z, TZ, H, and TH denote Z-order, transposed Z-order, Hilbert order, and transposed Hilbert order, respectively.
We compare EUGen, SNNK, and RaLiMP using the same PTv3 backbone and MLP replacement configurations.
For each $k \in \{5,10,22\}$, all three methods replace the MLP modules in the same Transformer blocks, ensuring that both the number and locations of replaced modules are matched across methods.
The backbone and RaLiMP configurations, together with the MLP replacement allocations, are summarized in Tab.~\ref{tab:ptv3-ralimp-config}.

We report mean intersection over union (mIoU), mean class accuracy (mAcc), and overall accuracy (allAcc).
mIoU and mAcc average the intersection over union and classification accuracy across semantic classes, respectively, while allAcc measures the fraction of correctly classified points.
The results for all three metrics are reported in Tab.~\ref{tab:mlp_replacement}.

\begin{table*}[h]
\caption{Semantic segmentation performance on indoor (ScanNet and ScanNet200)
and outdoor (nuScenes) benchmarks.
$k$ denotes the number of replaced MLP modules in PTv3.
}
\label{tab:mlp_replacement}
\centering
    \begin{adjustbox}{width=\textwidth}
        \begin{tabular}{ll|ccc|ccc|ccc}
        \toprule
        \multirow{2}{*}{$k$} & \multirow{2}{*}{Methods}
        & \multicolumn{3}{c|}{ScanNet Val}
        & \multicolumn{3}{c|}{ScanNet200 Val}
        & \multicolumn{3}{c}{nuScenes Val} \\
        \cmidrule(lr){3-5}\cmidrule(lr){6-8}
        \cmidrule(lr){9-11}
        & & mIoU & mAcc & allAcc
          & mIoU & mAcc & allAcc
          & mIoU & mAcc & allAcc \\
        \midrule
        \multirow{3}{*}{5}
        & EUGen         & 75.65 & 83.98 & 91.54 & 34.62 & 45.47 & 83.37 & 79.14 & 86.36 & 94.56 \\
        & SNNK          & 76.84 & 84.33 & 91.89 & 35.65 & 46.37 & 83.78 & 80.76 & 87.38 & 94.79 \\
        & RaLiMP (ours) & 78.08 & 85.43 & 92.37 & 36.27 & 46.54 & 83.65 & 80.81 & 87.20 & 94.77 \\
        \midrule
        \multirow{3}{*}{10}
        & EUGen         & 74.50 & 82.77 & 90.87 & 33.08 & 43.65 & 82.76 & 76.34 & 85.56 & 94.38 \\
        & SNNK          & 76.63 & 84.19 & 91.69 & 34.69 & 44.33 & 83.40 & 79.92 & 86.82 & 94.68 \\
        & RaLiMP (ours) & 77.54 & 85.24 & 92.09 & 35.50 & 45.36 & 83.24 & 80.65 & 87.19 & 94.75 \\
        \midrule
        \multirow{3}{*}{22}
        & EUGen         & 72.80 & 81.49 & 90.22 & 30.56 & 41.96 & 81.16 & 74.84 & 84.11 & 94.21 \\
        & SNNK          & 76.35 & 84.15 & 91.66 & 34.01 & 43.99 & 83.26 & 79.79 & 87.40 & 94.64 \\
        & RaLiMP (ours) & 77.33 & 84.89 & 91.93 & 34.90 & 45.62 & 83.41 & 80.47 & 87.13 & 94.65 \\
        \midrule
        \multicolumn{2}{l|}{Baseline (dense, full FT)}
        & 77.57 & 84.97 & 92.04 & 35.26 & 46.05 & 83.38 & 80.32 & 87.25 & 94.63 \\
        \bottomrule
        \end{tabular}
    \end{adjustbox}
\end{table*}
\begin{table*}[h]
\centering
\small
\caption{PTv3 backbone and RaLiMP configurations for semantic segmentation.
    $C$ denotes the channel dimension of the corresponding block.
    $\tau_0$ and $\lambda_{\mathrm{gate},0}$ denote the initial values of the learnable rational input scale and gate, respectively.
    $k$ denotes the number of replaced MLP modules, with replacement starting from the top MLP modules within each selected stage.
    }
\label{tab:ptv3-ralimp-config}
    \begin{tabular}{lc}
    \toprule
    \textbf{Configuration} & \textbf{Value} \\
    \midrule
    \multicolumn{2}{c}{\textit{\textbf{PTv3 backbone}}} \\
    \midrule
    Attention & Serialized softmax attention (FlashAttention) \\
    Serialization patterns & Z + TZ + H + TH \\
    Patch interaction & Shift Order + Shuffle Order \\
    Positional encoding & xCPE \\
    Encoder depths  & $[2,2,2,6,2]$ \\
    Encoder channels & $[32,64,128,256,512]$ \\
    Encoder heads & $[2,4,8,16,32]$ \\
    Encoder patch sizes & $[1024,1024,1024,1024,1024]$ \\
    Decoder depths & $[2,2,2,2]$ \\
    Decoder channels & $[64,64,128,256]$ \\
    Decoder heads & $[4,4,8,16]$ \\
    Decoder patch sizes & $[1024,1024,1024,1024]$ \\
    Downsampling strides & $[\times2,\times2,\times2,\times2]$ \\
    MLP ratio & 4 \\
    QKV bias & True \\
    Pre-normalization & True \\
    Drop path rate & 0.3 \\
    Backbone output channels & 64 \\
    \midrule
    \multicolumn{2}{c}{\textit{\textbf{RaLiMP configuration}}} \\
    \midrule
    Random feature dimension & 8 \\
    Random projections (base and rational blocks) & Frozen Gaussian \\
    Rational atom type & General shifted \\
    Number of rational atoms $J$ & 2 \\
    Rational quadrature frequencies $n_{\mathrm{freq}}$
    ($C<256$ / $C\geq256$) & $2/3$ \\
    Random projections per frequency $m$
    ($C<256$ / $C\geq256$) & $4/8$ \\
    Rational input normalization & Disabled \\
    Readout weights and rational coefficients & Learnable \\
    Rational input scale initialization $\tau_0$
    & 3.0 \\
    Gate initialization per block $\lambda_{\mathrm{gate},0}$
    & $\{0.5,1.0,2.0\}$ \\
    \midrule
    \multicolumn{2}{c}{\textit{\textbf{MLP replacement configurations}}} \\
    \midrule
    Allocation for $k=5$
    & Enc. $[0,0,0,3,1]$, Dec. $[0,0,0,1]$ \\
    Allocation for $k=10$
    & Enc. $[0,0,0,6,2]$, Dec. $[0,0,0,2]$ \\
    Allocation for $k=22$
    & Enc. $[2,2,2,6,2]$, Dec. $[2,2,2,2]$ \\
    \bottomrule
    \end{tabular}
\end{table*}
\clearpage

\subsection{Coordinate Networks}
\label{sec:exp_inr_exp}
Next we discuss the experimental details for NeRF, SIREN and PINN networks. 

\subsubsection{SIREN} 
Image fitting on Cameraman ($256 \times {256}$) with $3$ hidden layers of width $256$ and
$\omega_0=30$, trained by Adam on the full-image MSE for $500$ steps, seed $0$, each arm's learning
rate tuned over $10^{-4}$ to $10^{-2}$; we report $\mathrm{PSNR}=10\log_{10}(1/\mathrm{MSE})$. The
first sine layer is kept for all the methods and all of hidden layers are replaced by linear variants. A periodic target uses
$\cos/\sin$ pair at $n_{freq} = 1,\ \tau = 1$. Ours uses $128$ projections, hence $256$ features $[\cos,\sin]$ and $99$k deployed parameters against $198$k for SIREN. Both SNNK and EUGen use $256$ random features to match the same inference parameters as RaLiMP.
\begin{figure}[h]
    \centering
    \includegraphics[width=0.95\linewidth]{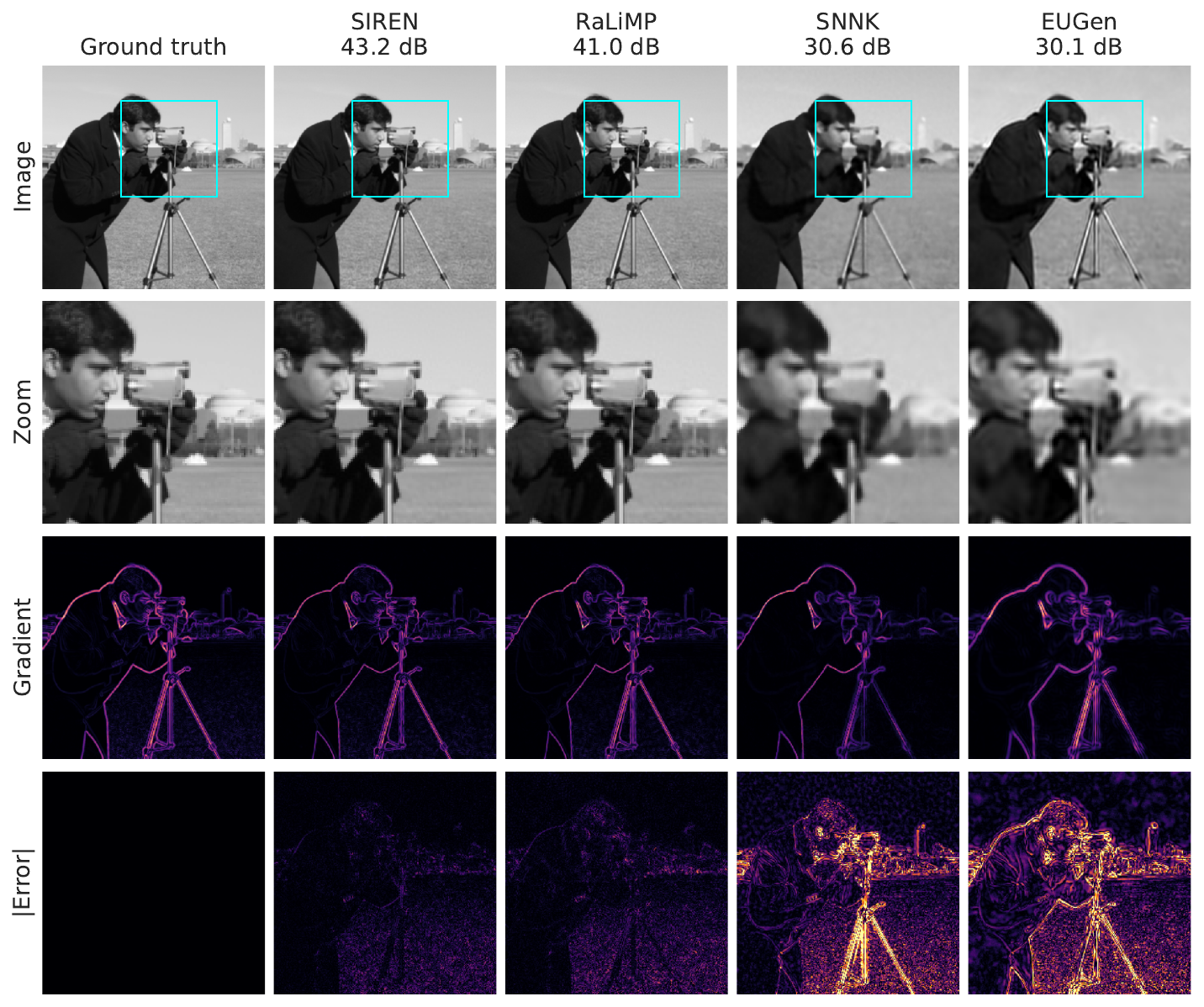}
    \caption{\small{\textbf{SIREN Networks.} RaLiMP matches the dense baseline while replacing all layers and outperform all variants. }}
    \label{fig:siren}
\end{figure}

\subsubsection{NeRF} 
\label{sec:exp_nerf_app}
On the $8$ synthetic Blender scenes at half resolution we follow the EUGen protocol in replacing the feed-forward layers with our layers including : depth $8$, width $256$, positional
encoding of order $10$ (and $4$ for view directions), $1024$ rays per step, $64$ coarse and $128$ fine
samples, Adam decaying by a factor $10$ every $500$k steps, for $200$k iterations, per-scene learning rate $5\!\times\!10^{-4}$ for chair, drums,
    lego, materials and mic and $10^{-4}$ for ficus, hotdog and ship. We report the mean PSNR over the test views. All the linear methods use matched $256$ random features. Moreover RaLiMP uses $n_{freq} =  1, \tau=1$. 

\subsubsection{Physics-Informed Network}
\label{sec:exp_pinn_app}
The viscous Burgers equation $u_t+uu_x=\nu u_{xx}$ with $\nu=0.01/\pi$ on $x\in[-1,1]$, $t\in[0,1]$,
$u(0,x)=-\sin(\pi x)$ and $u(t,\pm1)=0$. The network is $4$ hidden layers of width $128$ with $\tanh$
activations, trained by Adam (lr $10^{-3}$, cosine decay) for $20$k steps on $10$k collocation points and
$200$ initial/boundary points with an unweighted sum of the residual, initial and boundary losses; seed
$0$. We report relative $L_2$ error against a finite-difference reference on a $101\times256$ grid. We replace all hidden layers by our variants and the
replaced layers are constrained to fewer deployed parameters than the $\tanh$ network ($66.6$k): ours
$38.5$k ($K=3$, $8$ polynomial projections, $n_{freq} = 3$ nodes, $4$ projections per node, $\tau=3$, $s_{max} = 4$), EUGen
$38.5$k ($12$ projections per degree) and SNNK $38.4$k ($36$ random features).
\begin{figure}[h]
    \centering
    \includegraphics[width=0.95\linewidth]{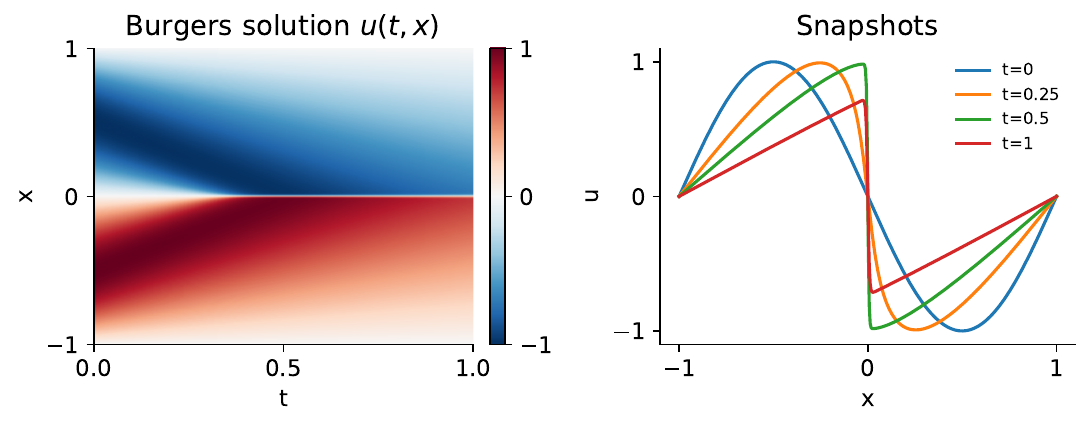}
    \includegraphics[width=.95\linewidth]{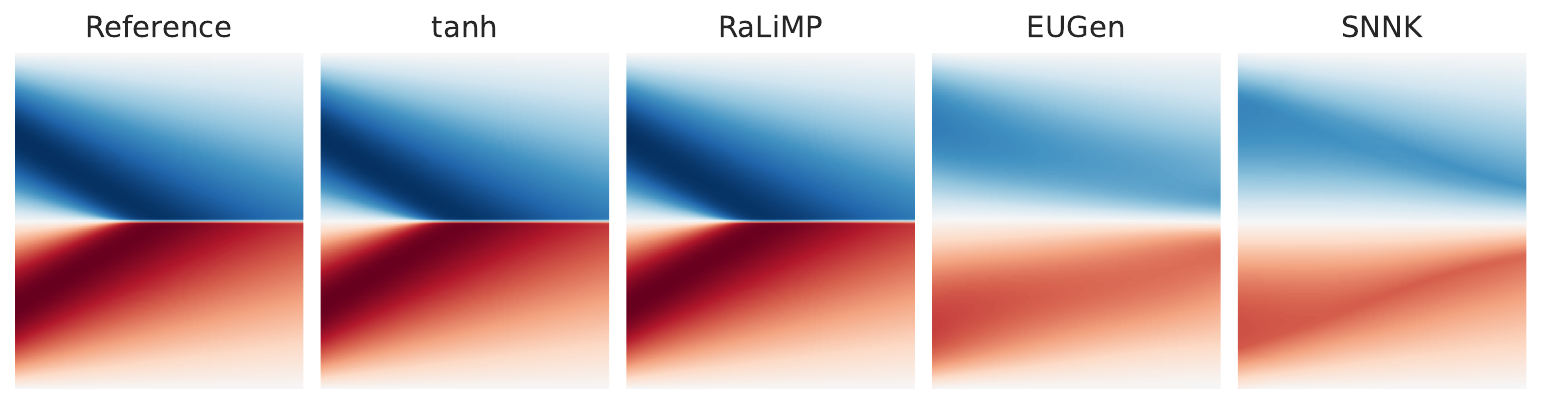}
    \caption{\textbf{Top Row : Reference solution of the viscous Burgers' equation.} \emph{Left:} the ground-truth field
    $u(t,x)$ over $(t,x)\in[0,1]\times[-1,1]$. \emph{Right:} snapshots at $t\in\{0,0.25,0.5,1\}$, showing the
    initial sinusoidal profile steepening into a sharp shock at $x{=}0$. \textbf{Bottom Row : Predicted Burgers' fields.} Solution $u(t,x)$ recovered by each model against the reference
    (leftmost). The dense tanh network and RaLiMP sharply reproduce the shock at $x{=}0$, whereas EUGen and SNNK
    smear it and lose the discontinuity, at matched inference budget. Tanh refer to the baseline PiNN model with tanh activation.}
    \label{fig:burgers}
\end{figure}

\subsection{Qwen3-8B experimental details}
\label{app:qwen}
In this section we describe in detail the Qwen experiments.

\subsubsection{Datasets}
First we list all the datasets that are used for this experiment.
\paragraph{Training.}
\begin{itemize}
\item \textbf{OpenThoughts3-1.2M} \citep{guha2025openthoughtsdatarecipesreasoning} : 1.2M long chain-of-thought reasoning traces; used for the supervised fine-tuning of the replacement layers (Section~\ref{app:qwen-sft}).
\item \textbf{DeepScaleR} \citep{DEEPSCALER} : 40{,}027 competition-mathematics problems with verifiable answers; the prompt pool from which the compressed model generates its own solutions for the Stage-2 on-policy rollouts (Section~\ref{app:qwen-rollouts}). After sampling and filtering to correct, self-terminating solutions, the retained Stage-3 distillation set contains $87{,}000$ (problem, solution) pairs for the $k = 8$ ($20k$ for the $k = 4$ model). The data consists of math problems with short prompts (median $63$ tokens, $10$th/$90$th percentiles $30$/$129$) and long chain-of-thought solutions (median $6{,}538$ tokens, $10$th/$90$th percentiles $2{,}271$/$13{,}042$), of which $0.9\%$ reach the $16{,}384$-token cap. We additionally pool a \emph{science-focused} data with the mathematics pool above. The science prompts come from
\textbf{WebInstruct-verified} \citep{yue2024mammoth2}, restricted to the questions tagged physics, chemistry,
or biology whose answers are numeric or short symbolic expressions (answer types \emph{Float}, \emph{Integer},
or \emph{Expression}); in practice these are chiefly physics and chemistry. This yields $2{,}287$ traces over $2{,}256$ distinct problems,
with similarly short prompts (median $58$ tokens) but somewhat shorter solutions (median $4{,}098$ tokens,
$10$th/$90$th percentiles $1{,}522$/$7{,}478$). 
\end{itemize}

\paragraph{Evaluation.} ($n$ denotes the number of problems.)
\begin{itemize}
\item \textbf{AMC23} ($n{=}83$) \citep{AMC23} : problems from the 2023 American Mathematics
      Competitions (AMC 10/12); avg@16 and pass@16.\footnote{We use the $83$-problem
      \texttt{AI-MO/aimo-validation-amc} release.}
\item \textbf{AIME24} ($n{=}30$) \citep{aime_2024} : the 2024 American Invitational Mathematics Examination; avg@16 and pass@16.
\item \textbf{AIME25} ($n{=}30$) \citep{aime_2025} : the 2025 American Invitational Mathematics Examination; avg@16 and pass@16.
\item \textbf{GSM8K} ($n{=}1{,}319$ test) \citep{cobbe2021gsm8k} : grade-school arithmetic word problems.
\item \textbf{MATH500} ($n{=}500$) \citep{hendrycks2021math} : a 500-problem subset of the MATH competition-mathematics benchmark (algebra, geometry, number theory, precalculus).
\item \textbf{Minerva Math} ($n{=}272$) \citep{MINERVA} : undergraduate-level quantitative problems in physics, chemistry, and related sciences.
\item \textbf{OlympiadBench} ($n{=}674$) \citep{OLYMPIADBENCH} : olympiad-level mathematics and
      physics problems.
\item \textbf{GPQA-Diamond} ($n{=}198$) \citep{GPQA} : graduate-level,  multiple-choice questions in biology, physics, and chemistry, avg@16 and pass@16.
\item \textbf{MMLU} ($n{=}14{,}042$ test) \citep{mmlu} : multiple-choice questions across 57 subjects spanning STEM, humanities, and social sciences.
\end{itemize}

We compress Qwen3-8B by replacing the feed-forward network in the top $k$ transformer blocks with our
random-feature layer, and then recover the lost accuracy in two training stages: (i) a supervised
fine-tuning stage that fits the replacement layers on a large reasoning corpus, and (ii) an on-policy
distillation stage in which the compressed model learns to reproduce the original model's next-token
distribution on the compressed model's own generated solutions. We run the identical pipeline at two
compression levels, $k=8$ and $k=4$; the only differences between them are which blocks are replaced  (Table~\ref{tab:qwen-arch}). Everything else, i.e. optimizer, schedule, and the two-stage procedure is shared. The k = 4 model requires less data in both stages to recover its performance. 

Qwen3-8B has 36 transformer blocks. For $k=8$ we replace the feed-forward network in blocks 28--35 (the
last eight); for $k=4$ we replace blocks 32--35 (the last four). In both cases the attention layers, the
embedding, and the output head are left unchanged, and the original feed-forward weights of the replaced
blocks are discarded. During all training the original (non-replaced) weights are kept frozen and adapted
only through low-rank adapters (LoRA)~\citep{lora}; the replacement layers themselves are trained fully.

The replacement layer is exactly the construction of Section~\ref{sec:hadamard}, applied to
the gated (SwiGLU) feed-forward network $\mathrm{SiLU}(\mathbf{Wx})\odot(\mathbf{Ux})$: the SiLU gate is expressed through the
random-feature map (linear and squared-dot terms plus the rational term), and the gated product is collapsed to a rank-$r$ factorization. Table~\ref{tab:qwen-arch} lists all the settings used in our experiments.

The feature map has width $672$: a linear group of $192$; a squared-dot group of $192$ (two projections
$\mathbf{G},\mathbf{H}$ combined by an element-wise product, which does not double the width); and a rational group of
$2\times 3\times 48 = 288$ ($3$ retained quadrature frequencies, $48$ features each, a cosine and sine tower
per frequency).

We use a single pole $1/(1+t^2)$, with $\tau=1$, and
normalize the pre-activation onto the unit sphere before the rational projection, and as a further safeguard the exponent of the rational random features (the argument of the $e^{(\cdot)}$ factor in Eq.~\eqref{eq:ratfeat}) is clamped from above at $10$. The gated product is collapsed to
Hadamard rank $r=512$ for $k=8$; for $k=4$ we report both $r=256$ and $r=512$. The SiLU coefficients are
held at their fixed values, and all random projections together with the rational weight and node
coefficients are trained.
\begin{table}[h]
\centering
\caption{Architecture of the replaced feed-forward layers. Both settings replace the top $k$ blocks of
Qwen3-8B and share all other hyperparameters.}
\label{tab:qwen-arch}
\begin{tabular}{lcc}
\toprule
 & $k=8$ & $k=4$ \\
\midrule
Blocks replaced                                  & 28--35 & 32--35 \\
Linear (random-projection) group dimension       & 192    & 192 \\
Quadratic (EUGen squared-dot) group dimension    & 192    & 192 \\
Rational quadrature frequencies                  & 6      & 6 \\
Rational random features per frequency           & 48     & 48 \\
Rational group width ($2\times 3\times 48$)      & 288    & 288 \\
Total feature width                              & 672    & 672 \\
Linear coefficient (fixed)                       & 0.5    & 0.5 \\
Quadratic coefficient (fixed)                    & 0.25   & 0.25 \\
Rational weight (learnable, init.)               & 0.0    & 0.0 \\
$\tau$ & 1 & 1 \\
Cap on rational-feature exponent (Eq.~\ref{eq:ratfeat})  & 10.0   & 10.0 \\
Hadamard rank $r$                                & 512    &  512 \\
\bottomrule
\end{tabular}
\end{table}

\subsubsection{Stage 1: supervised fine-tuning of the replacement layers}
\label{app:qwen-sft}

We first fit the freshly-initialized replacement layers to the base model's behavior by supervised
fine-tuning on OpenThoughts3-1.2M~\citep{guha2025openthoughtsdatarecipesreasoning}. We
stream the corpus shard by shard and tokenize on the fly, so that so no example repeats within a pass, with a fixed seed for the shard order and sampling.
 Each training example is a (problem, solution) pair; the loss is the standard next-token cross-entropy on the solution tokens. We choose the best checkpoint based on validation accuracy on a held out portion of OpenThoughts to mimic the distribution of the downstream tasks. Note that fine-tuning base Qwen model on this dataset using LoRA does not improve the base model. 

\paragraph{Hyperparameters.}
\begin{itemize}
\item Optimizer: AdamW, learning rate $1\times10^{-4}$, cosine decay to $10\%$ of the peak, 100 warmup steps.
\item Weight decay: 0.
\item Sequence length: 16{,}384 tokens; examples longer than this are dropped so that only complete
      (end-of-sequence terminated) traces are trained on.
\item Batch size: 1 sequence per step, with gradient accumulation over 16 micro-batches (effective batch
      size 16 sequences).
\item Number of optimization steps: 64{,}000 for k = 8 and 40{,}000 for k = 4.
\item Frozen backbone with low-rank adapters on all linear layers (attention and feed-forward projections):
      LoRA rank 32, LoRA scaling factor 64. The replacement layers are trained fully.
\item Attention implementation: scaled dot-product attention. Activation checkpointing enabled.
\item Precision: bfloat16.
\end{itemize}
This stage produces the compressed model we refer to as the \emph{pre-distillation} model in the main text.

\subsubsection{Stage 2: sampling the compressed model's own solutions}
\label{app:qwen-rollouts}

On-policy distillation requires solutions drawn from the \emph{compressed} model itself so
that the teacher is later scored on states the compressed model actually visits. We therefore generate a
pool of solutions from the Stage-1 model on Competition-mathematics problems from DeepScaleR. Each problem is presented with the instruction ``Please reason step
by step, and put your final answer within \verb|\boxed{}|.'' For each problem we draw 8 independent samples with temperature 0.8, nucleus sampling $p=0.95$, top-$k$ = 20, and a maximum generation length of 16{,}384 tokens.

\paragraph{Filtering.} We keep only solutions that (i) terminate on their own (emit the end-of-sequence
token before the length cap) and (ii) are correct, verified by extracting the final \verb|\boxed{}| answer
and checking it against the ground-truth answer. The kept (problem, solution) pairs form the on-policy
distillation set. This pool contains roughly 87{,}000 solutions from math datasets and about 2700 from science. Notably for the k = 4 model, we use about 18{,}000 traces from the math dataset as that model was already close to base after fine-tuning. 

\subsubsection{Stage 3: on-policy distillation}
\label{app:qwen-distill}

Starting from the Stage-1 compressed model, we distill the base Qwen3-8B into the compressed
model on the solutions collected in Stage 2. Both the compressed (``student'') model and the frozen
original (``teacher'') model process the same solution; the training loss is the forward KL divergence
$\mathrm{KL}\!\left(p_\text{teacher} \,\|\, p_\text{student}\right)$ between the teacher's and student's
raw next-token distributions (temperature = 1), averaged over the solution tokens. It is on-policy because these solutions are the student's own Stage-2 generations, so the teacher is scored on the states the student actually visits.

\paragraph{Hyperparameters.}
\begin{itemize}
\item Initialization: the Stage-1 compressed model (Section~\ref{app:qwen-sft}).
\item Distillation data: the on-policy solution pool from Stage~2 (Section~\ref{app:qwen-rollouts}).
\item Optimizer: AdamW, learning rate $2\times10^{-5}$, cosine decay to $10\%$ of the peak, 10 warmup steps.
\item Weight decay: 0.
\item Forward KL divergence as the only loss term (weight 1.0).
\item The forward KL is computed over the completion tokens in chunks of 2{,}048 tokens to bound memory,
      then summed over chunks and averaged over tokens.
\item Sequence length: 12{,}288 tokens.
\item Batch size: 1 sequence per step, with gradient accumulation over 24 micro-batches (effective batch
      size 24 sequences).
\item Frozen backbone with low-rank adapters on all linear layers: LoRA rank 32, LoRA scaling factor 64.
      The replacement layers are trained fully.
\item Attention implementation: scaled dot-product attention. Activation checkpointing enabled.
\item Precision: bfloat16. The frozen teacher is held on a separate device from the student.
\item Number of optimization steps: the distillation set is trained for two epochs. We select the
      final checkpoint by validation accuracy on a small 150 sample in Numina-Math train~\citep{li2024numinamath}.
\end{itemize}

The $k=4$ model is produced by the identical Stage-3 procedure, initialized from the $k=4$ Stage-1 model,
with the only architectural differences being those in Table~\ref{tab:qwen-arch}.

\subsubsection{Evaluation protocol}
\label{app:qwen-eval}

All models are evaluated with temperature 0.6, nucleus sampling $p=0.95$, top-$k$ 20, and a maximum
generation length of 16{,}384 tokens, in the model's default reasoning (``thinking'') mode.

Since AMC, AIME and GPQA are small datasets, we therefore draw
16 samples per problem and report the mean per-problem accuracy (the average over the 16 samples), which we
denote avg@16. We also report pass@16 (a problem counts as solved if any of the 16 samples is correct) as a capability ceiling. For the larger benchmarks (GSM8K, MATH500, Minerva, OlympiadBench) we report a single sample per problem at the same decoding settings.

Mathematical answers are extracted from the last \verb|\boxed{}| in the generation and scored against the ground truth by \texttt{math\_verify} for symbolic/numeric equivalence, with a boxed-string match and a numeric-tolerance check as fallbacks when parsing fails; the same check gates the Stage-2 filter. The single-sample benchmarks decode with a
fixed seed, while the $16$ avg@16 draws are independent. The multiple-choice answers are extracted as the boxed letter. MMLU is evaluated via log probabilities via~\citep{eval-harness}.

\subsubsection{Effect of the distillation stage}
\label{app:qwen-recovery}
Table~\ref{tab:qwen-recovery} isolates Stage 3 (on-policy distillation) at $k=8$, with benchmarks ordered by
base generation length $L$. The Stage-1 supervised model loses accuracy that grows with $L$. The
length-compounding pattern of Eq.~\ref{eq:delta_law} and distillation recovers most of it, with the
largest gains on the longest-generation benchmarks. After it, GSM8K and MATH500 exceed the
uncompressed model and Minerva and OlympiadBench are within a point.

\begin{table}[h]
\centering
\caption{Effect of the on-policy distillation stage (accuracy, \%), benchmarks ordered by base median
generation length $L$ (tokens). For each compression level, ``pre'' is the Stage-1 supervised model and
``\,$+$dist'' is the same model after Stage 3; ``Base'' is uncompressed Qwen3-8B. AMC/AIME are avg@16, the
rest single-sample; maximum generation length 16{,}384 tokens.}
\label{tab:qwen-recovery}
\begin{tabular}{lc cc cc c}
\toprule
& & \multicolumn{2}{c}{$k{=}4$} & \multicolumn{2}{c}{$k{=}8$} & \\
\cmidrule(lr){3-4}\cmidrule(lr){5-6}
Benchmark & $L$ & pre & $+$dist & pre & $+$dist & Base \\
\midrule
GSM8K         & 1{,}685  & 94.5 & 95.5 & 92.3 & 95.2 & 93.5 \\
MATH500       & 3{,}528  & 87.4 & 92.8 & 85.2 & 94.5 & 90.0 \\
Minerva       & 4{,}892  & 35.3 & 45.3 & 28.7 & 43.4 & 45.6 \\
AMC23         & 7{,}820  & 75.2 & 82.2 & 71.8 & 78.8 & 83.0 \\
OlympiadBench & 8{,}244  & 65.4 & 70.4 & 63.8 & 69.0 & 69.9 \\
AIME24        & 12{,}298 & 51.0 & 61.0 & 42.3 & 58.1 & 67.5 \\
GPQA          & 12{,}803 & 41.4 & 52.4 & 28.8 & 51.1 & 56.1 \\
AIME25        & 15{,}942 & 43.3 & 52.3 & 34.2 & 48.6 & 57.7 \\
\bottomrule
\end{tabular}
\end{table}

\subsubsection{Length-compounding of the kernel approximation error}
\label{sec:delta_law}
A finite-feature kernel block incurs a bounded per-token approximation error that perturbs the output logits once per decoded token, so its effect on end-to-end accuracy compounds with the generation length $L$. Modeling this as a constant per-token derailment rate $\delta$ gives a multiplicative law
\begin{equation}
\mathrm{acc}\;\approx\;\mathrm{acc}_{\mathrm{base}}\exp(-\delta L),
\qquad
\text{i.e.}\quad
\text{deficit}\;=\;1-\frac{\mathrm{acc}}{\mathrm{acc}_{\mathrm{base}}}\;\approx\;1-e^{-\delta L}.
\label{eq:delta_law}
\end{equation}
Figure~\ref{fig:delta_law} evaluates this for the pre-distillation model of Table~\ref{tab:qwen-recovery}. The deficit is monotone in $L$, from $3\%$ at $L\!\approx\!1.5$k (GSM8K) to $41\%$ at $L\!\approx\!16$k (AIME25). The implied rate $\delta$ is stable at $\sim\!2$--$4\times10^{-5}$ per token wherever the base model is unsaturated where the $\delta$ is dominated by ceiling effects. Two orderings are robust: $\delta$ grows with the swapped-layer count ($\delta_{k=8}>\delta_{k=4}$) and shrinks with feature rank at fixed $k$ ($r{=}512<r{=}256$), so $\delta$ tracks the kernel's approximation error rather than layer count alone. Long-chain benchmarks are thus \emph{predictably} harder for a fixed kernel budget, and the distillation stage above preferentially removes this length-dependent error.

\begin{figure}[h]
\centering
\includegraphics[width=0.62\linewidth]{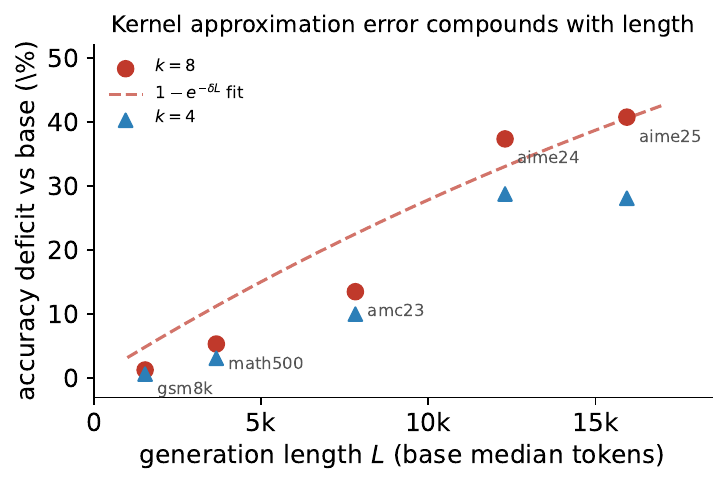}
\caption{\textbf{Accuracy deficit vs.\ generation length $L$ at matched decoding.} Points are the benchmarks; the dashed curve is the one-parameter fit $1-e^{-\delta L}$ ($\delta{=}3.2\times10^{-5}$, fit on the unsaturated benchmarks). Swapping more layers ($k{=}8$ vs.\ $k{=}4$) raises the deficit at fixed $L$.}
\label{fig:delta_law}
\end{figure}

\subsubsection{Effect of the Hadamard rank $r$}
\label{app:qwen-rank}
Table~\ref{tab:qwen-rank} compares the two $k{=}4$ variants, $r{=}256$ and $r{=}512$. Raising the rank
lowers the kernel's factorization error and improves the competition-math benchmarks. The
main-body table reports the $r{=}512$ variant.

\begin{table}[h]
\centering
\caption{$k{=}4$ at Hadamard rank $r{=}256$ vs.\ $r{=}512$ (accuracy after stage 1 training, \%; AMC/AIME are avg@16).}
\label{tab:qwen-rank}
\begin{tabular}{lccccccc}
\toprule
$r$ & AMC23 & AIME24 & AIME25 & GSM8K & MATH500 & Minerva & GPQA \\
\midrule
$256$ & 74.7 & 48.1 & 41.5 & 92.9 & 87.2 & 37.9 & 44.9 \\
$512$ & 75.2 & 51.0 & 43.3 & 94.5 & 87.4 & 35.3 & 41.4 \\
\bottomrule
\end{tabular}
\end{table}

\subsubsection{Memory footprint and throughput}
\label{app:qwen-efficiency}

Because the replacement layer folds its weight side into a single $[d\times r]$ matrix at inference, each
replaced feed-forward network is far smaller than the original gated network, and the saving grows linearly
with the number of blocks replaced. Table~\ref{tab:qwen-efficiency} reports the parameter reduction relative
to the full 8.19B-parameter model (the feed-forward networks account for about two-thirds of the model's
parameters). At $k=8$ the replacement removes 1.17B parameters (14\% of the whole model, 21\% of the
feed-forward parameters); at $k=4$ it removes 0.59B (7\% of the model).

\begin{table}[h]
\centering
\caption{Efficiency of replacing the top $k$ feed-forward networks of Qwen3-8B (36 blocks, 8.19B
parameters). Parameter reduction is relative to the full model / to the feed-forward parameters (which are
$\sim$two-thirds of the model). Throughput is compiled decoding (bfloat16; 4096-token prompt, 1024 generated
tokens, 16 concurrent sequences, single GPU).}
\label{tab:qwen-efficiency}
\begin{tabular}{lccccc}
\toprule
Model & \makecell{Params\\removed} & \makecell{\% of\\model} & \makecell{\% of\\FFN} & \makecell{Throughput\\(tok/s)} & \makecell{Speedup\\vs.\ base} \\
\midrule
Uncompressed & ---     & ---     & ---     & 763.3 & --- \\
$k=4$        & 0.59B   & 7.2\%   & 10.8\%  & 785.6 & $+2.9\%$ \\
$k=8$        & 1.17B   & 14.3\%  & 21.5\%  & 809.6 & $+6.1\%$ \\
\bottomrule
\end{tabular}
\end{table}

\section{Detailed Results}
In this section we present detailed results for NeRF, ViT, PTv3 and Causal Language Modeling.
\begin{table}[h]
\centering
\caption{\small{Synthetic-NeRF held-out test PSNR (dB). RaLiMP beats all other methods, at \emph{matched deployed parameters} and even beating the dense NeRF baseline.}}
\renewcommand{\arraystretch}{1.2}
\setlength{\tabcolsep}{6pt}
\resizebox{0.8\textwidth}{!}{%
\begin{tabular}{lccccccccc}
\toprule
Method & Chair & Drums & Ficus & Hotdog & Lego & Mater. & Mic & Ship & Avg \\
\midrule
NeRF      & 33.70 & 25.46 & 26.40 & 35.81 & 31.31 & 29.00 & 33.11 & 28.79 & 30.45 \\
EUGen    & 33.09 & 25.18 & 28.40 & 35.47 & 30.60 & 28.88 & 32.63 & 28.66 & 30.36 \\
SNNK & 30.66 & 24.76 & 28.21 & 35.72 & 30.00 & 28.64 & 30.44 & 28.41 & 29.61 \\
\midrule
RaLiMP (ours)  & 33.19 & 25.26 & 28.89 & 35.73 & 30.95 & 29.02 & 32.86 & 28.95 & 30.61 \\
\bottomrule
\label{table:nerf}
\end{tabular}
}
\vspace{-8mm}
\end{table}

\begin{table}[t]
\centering
\small
\caption{ViT results on CIFAR-100, ImageNet, and Places365 (Top-1 Accuracy). RaLiMPs beat all other methods at \emph{matched deployed inference budget}. Baseline is the dense ViT-B/16 full
fine-tune.}
\setlength{\tabcolsep}{8pt}
\begin{tabular}{llcccc}
\toprule
&& \multicolumn{4}{c}{$k$ (layers replaced)}\\
Dataset & Method & 4 & 7 & 9 & 11\\
\midrule
\multirow{4}{*}{CIFAR-100}
& Low Rank & 86.6 & 78.3 & 74.3 & 69.3\\
& SNNK & 87.3 & 84.0 & 80.3 & 75.9\\
& EUGen & 88.6 & 86.5 & 81.3 & 74.6\\
& \textbf{RaLiMP (ours)} & \textbf{89.8} & \textbf{88.1} & \textbf{84.6} & \textbf{80.4}\\
& Baseline & \multicolumn{4}{c}{92.0}\\
\midrule
\multirow{4}{*}{Places365}
& Low Rank & 54.1 & 52.8 & 49.7 & 45.9\\
& SNNK & 57.1 & 56.5 & 56.3 & 54.6\\
& EUGen & 56.9 & 56.9 & 56.1 & 54.2\\
& \textbf{RaLiMP (ours)} & \textbf{58.0} & \textbf{57.5} & \textbf{56.9} & \textbf{55.2}\\
& Baseline & \multicolumn{4}{c}{58.2}\\
\midrule
\multirow{4}{*}{ImageNet-1k}
& Low Rank & 78.6 & 72.4 & 71.0 & 62.8\\
& SNNK & 80.4 & 77.4 & 75.0 & 73.1\\
& EUGen & 79.1 & 77.2 & 75.8 & 72.8\\
& \textbf{RaLiMP (ours)} & \textbf{80.8} & \textbf{78.5} & \textbf{76.8} & \textbf{73.9}\\
& Baseline & \multicolumn{4}{c}{81.9}\\
\bottomrule
\end{tabular}
\label{table:vit-main-results}
\end{table}

\begin{table}[h]
\centering
\caption{VTAB-1k top-1 accuracy (\%), matched inference budget, per $k$ (layers replaced).
\textbf{R}aLiMP (ours), \textbf{S}NNK, \textbf{E}UGen, \textbf{L}owrank. \textbf{Bold} = best method at each
$(\text{task},k)$.}
\footnotesize
\renewcommand{\arraystretch}{1.15}
\setlength{\tabcolsep}{3pt}
\resizebox{0.9\textwidth}{!}{%
\begin{tabular}{l cccc c cccc c cccc c cccc}
\toprule
 & \multicolumn{4}{c}{$k{=}4$} & & \multicolumn{4}{c}{$k{=}7$} & & \multicolumn{4}{c}{$k{=}9$} & & \multicolumn{4}{c}{$k{=}11$} \\
\cmidrule(lr){2-5}\cmidrule(lr){7-10}\cmidrule(lr){12-15}\cmidrule(lr){17-20}
Task & R & S & E & L & & R & S & E & L & & R & S & E & L & & R & S & E & L \\
\midrule
\multicolumn{20}{l}{\emph{Natural}}\\
CIFAR-100     & \textbf{58.2} & 52.9 & 38.2 & 4.8 & & \textbf{57.5} & 43.3 & 26.7 & 3.3 & & \textbf{51.0} & 36.0 & 13.8 & 2.5 & & \textbf{40.3} & 29.1 & 3.6 & 2.8 \\
Caltech101    & \textbf{85.2} & 82.1 & 74.3 & 25.6 & & \textbf{84.4} & 79.0 & 63.6 & 4.2 & & \textbf{76.0} & 75.9 & 49.1 & 9.0 & & \textbf{69.4} & 69.2 & 7.0 & 3.8 \\
DTD           & \textbf{69.0} & 66.2 & 59.9 & 18.2 & & \textbf{67.9} & 63.5 & 53.4 & 8.5 & & \textbf{59.1} & 59.0 & 44.2 & 10.5 & & \textbf{50.0} & 49.7 & 22.8 & 9.7 \\
Flowers102    & \textbf{95.6} & 95.4 & 90.3 & 10.2 & & \textbf{94.3} & 94.0 & 76.9 & 6.0 & & \textbf{92.3} & 91.8 & 73.6 & 4.3 & & 85.2 & \textbf{86.1} & 31.2 & 6.7 \\
Pets          & \textbf{85.7} & 83.8 & 78.2 & 22.0 & & \textbf{85.0} & 79.7 & 65.3 & 10.7 & & \textbf{75.9} & 74.3 & 52.4 & 5.2 & & \textbf{66.7} & 64.0 & 2.8 & 8.7 \\
SVHN          & \textbf{89.3} & 88.0 & 87.3 & 66.7 & & \textbf{89.2} & 88.2 & 84.9 & 68.0 & & \textbf{89.7} & 89.5 & 31.4 & 22.1 & & \textbf{85.8} & 85.7 & 19.6 & 34.5 \\
Sun397        & \textbf{46.5} & 35.2 & 22.4 & 4.5 & & \textbf{44.8} & 28.9 & 15.3 & 3.0 & & \textbf{35.2} & 24.7 & 10.6 & 3.4 & & \textbf{30.1} & 20.0 & 3.8 & 2.5 \\
\midrule
\multicolumn{20}{l}{\emph{Specialized}}\\
PatchCamelyon & \textbf{88.7} & 87.5 & 87.1 & 83.3 & & \textbf{88.5} & 87.1 & 87.4 & 82.1 & & \textbf{88.4} & 88.2 & 82.8 & 84.7 & & \textbf{86.0} & 85.8 & 78.6 & 77.1 \\
EuroSAT       & \textbf{95.3} & 94.7 & 94.5 & 91.7 & & \textbf{95.7} & 95.4 & 94.3 & 71.7 & & \textbf{95.6} & 95.4 & 91.6 & 48.9 & & \textbf{93.4} & 93.4 & 82.1 & 70.6 \\
Resisc45      & \textbf{85.8} & 84.3 & 81.5 & 16.5 & & \textbf{85.2} & 84.8 & 78.2 & 8.2 & & \textbf{84.0} & 82.6 & 72.2 & 17.8 & & \textbf{75.6} & 75.3 & 37.3 & 19.6 \\
Retinopathy   & \textbf{78.2} & 76.7 & 76.0 & 76.9 & & \textbf{76.8} & 76.5 & 76.6 & 76.0 & & \textbf{77.6} & 76.7 & 73.8 & 76.3 & & \textbf{75.5} & 75.3 & 73.6 & 73.6 \\
\midrule
\multicolumn{20}{l}{\emph{Structured}}\\
Clevr-count   & \textbf{67.8} & 51.8 & 56.0 & 66.5 & & \textbf{70.2} & 56.9 & 63.9 & 68.5 & & \textbf{68.3} & 53.4 & 67.2 & 58.6 & & \textbf{57.8} & 53.2 & 53.9 & 48.7 \\
Clevr-dist    & 54.7 & 52.0 & 54.1 & \textbf{55.1} & & \textbf{58.0} & 52.9 & 56.4 & 54.6 & & \textbf{57.6} & 55.5 & 55.9 & 55.7 & & \textbf{56.5} & 55.4 & 56.2 & 54.2 \\
DMLab         & \textbf{51.5} & 49.4 & 51.5 & 40.5 & & \textbf{50.9} & 50.3 & 47.0 & 38.3 & & \textbf{49.4} & 47.3 & 39.4 & 38.9 & & \textbf{45.6} & 45.3 & 35.7 & 35.1 \\
dSprites-loc  & \textbf{44.1} & 41.4 & 43.9 & 37.1 & & \textbf{44.9} & 42.1 & 43.9 & 37.4 & & \textbf{43.5} & 40.0 & 42.1 & 8.8 & & \textbf{45.3} & 41.4 & 19.7 & 35.2 \\
dSprites-ori  & \textbf{29.8} & 29.7 & 29.4 & 22.9 & & \textbf{30.7} & 29.8 & 27.4 & 13.2 & & \textbf{28.8} & 28.2 & 13.6 & 7.6 & & \textbf{27.4} & 27.2 & 7.6 & 10.7 \\
KITTI-dist    & \textbf{52.5} & 51.6 & 51.0 & 41.6 & & \textbf{53.7} & 51.9 & 51.9 & 46.4 & & 54.8 & \textbf{54.9} & 43.3 & 49.0 & & \textbf{52.0} & 50.1 & 47.7 & 49.5 \\
sNORB-azim    & \textbf{26.3} & 24.4 & 25.4 & 10.1 & & \textbf{26.2} & 25.8 & 21.0 & 7.1 & & \textbf{27.9} & 26.9 & 5.6 & 6.5 & & \textbf{26.2} & 26.1 & 7.8 & 8.5 \\
sNORB-elev    & \textbf{41.0} & 40.6 & 37.2 & 37.0 & & \textbf{40.4} & 38.0 & 40.1 & 15.5 & & \textbf{41.5} & 38.5 & 40.7 & 30.1 & & \textbf{32.0} & 31.8 & 25.8 & 23.1 \\
\bottomrule
\end{tabular}
}
\label{table-vtab}
\end{table}

\clearpage
\begin{table*}[h]
\centering
    \begin{minipage}[t]{0.48\linewidth}
        \centering
        \caption{Semantic segmentation results on ScanNet v2. Val denotes validation mIoU. Bold indicates the best result in the table.}
        \label{tab:scannetv2}
        \small
        \setlength{\tabcolsep}{3pt}
        \renewcommand{\arraystretch}{0.95}
        \begin{tabularx}{\linewidth}{@{}Yc@{}}
            \toprule
            \textbf{Method} & \textbf{Val} \\
            \midrule
            \multicolumn{2}{@{}c@{}}{\textit{\textbf{MLP/CNN}}} \\
            \midrule
            PointNet++~\citep{qi2017pointnet++}        & 53.5 \\
            SparseConvNet~\citep{graham20183d}         & 69.3 \\
            PointConv~\citep{wu2019pointconv}          & 61.0 \\
            JointPointBased~\citep{chiang2019unified}  & 69.2 \\
            KPConv~\citep{thomas2019kpconv}            & 69.2 \\
            PointASNL~\citep{yan2020pointasnl}         & 63.5 \\
            PointNeXt~\citep{qian2022pointnext}        & 71.5 \\
            LargeKernel3D~\citep{chen2022scaling}      & 73.5 \\
            PointMetaBase~\citep{lin2023meta}          & 72.8 \\
            MinkowskiNet~\citep{choy20194d}            & 72.2 \\
            \midrule
            \multicolumn{2}{@{}c@{}}{\textit{\textbf{Transformer}}} \\
            \midrule
            Fast Point Transformer~\citep{park2022fast} & 72.1 \\
            Stratified Transformer~\citep{lai2022stratified} & 74.3 \\
            PointConvFormer~\citep{wu2023pointconvformer} & 74.5 \\
            OctFormer~\citep{wang2023octformer}        & 75.7 \\
            Swin3D~\citep{yang2025swin3d}              & 76.4 \\
            PTv1~\citep{zhao2021point}                 & 70.6 \\
            PTv2~\citep{wu2022point}                   & 75.4 \\
            PTv3~\citep{wu2024point}                   & 77.6 \\
            \midrule
            \multicolumn{2}{@{}c@{}}{\textit{\textbf{Mamba}}} \\
            \midrule
            Point Mamba~\citep{liu2024point}           & 74.6 \\
            Pamba~\citep{li2025pamba}                  & 77.6 \\
            \midrule
            \multicolumn{2}{c}{\textit{\textbf{PTv3 variants (our experiments)}}} \\
            \midrule
            MLP (dense baseline) & 77.6 \\
            RaLiMP ($k=5$)       & \textbf{78.1} \\
            RaLiMP ($k=10$)      & 77.6 \\
            RaLiMP ($k=22$)      & 77.3 \\
            \bottomrule
        \end{tabularx}
    \end{minipage}
    \hfill
    \begin{minipage}[t]{0.48\linewidth}
        \centering
        \caption{Semantic segmentation results on nuScenes. Val denotes validation mIoU. Bold indicates the best result in the table.}
        \label{tab:nuscenes_semseg}
        \small
        \setlength{\tabcolsep}{3pt}
        \renewcommand{\arraystretch}{0.95}
        \begin{tabularx}{\linewidth}{@{}Yc@{}}
            \toprule
            \textbf{Method} & \textbf{Val} \\
            \midrule
            \multicolumn{2}{@{}c@{}}{\textit{\textbf{MLP/CNN/SparseConv}}} \\
            \midrule
            SPVNAS~\citep{tang2020searching} & 77.4 \\
            Cylinder3D~\citep{zhu2021cylindrical} & 76.1 \\
            PVKD~\citep{hou2022point} & 76.0 \\
            2DPASS~\citep{yan20222dpass} & 79.4 \\
            MinkUNet~\citep{choy20194d} & 73.3 \\
            ~~+ M3Net~\citep{liu2024multi} & 79.0 \\
            ~~+ PPT~\citep{wu2024towards} & 78.6 \\
            OA-CNNs~\citep{peng2024oa} & 78.9 \\
            \midrule
            \multicolumn{2}{@{}c@{}}{\textit{\textbf{Transformer}}} \\
            \midrule
            SphereFormer~\citep{lai2023spherical} & 78.4 \\
            RangeFormer~\citep{kong2023rethinking} & 78.1 \\
            PTv2~\citep{wu2022point} & 80.2 \\
            PTv3~\citep{wu2024point} & 80.4 \\
            \midrule
            \multicolumn{2}{@{}c@{}}{\textit{\textbf{Mamba}}} \\
            \midrule
            Pamba~\citep{li2025pamba} & 80.4 \\
            \midrule
            \multicolumn{2}{c}{\textit{\textbf{PTv3 variants (our experiments)}}} \\
            \midrule
            MLP (dense baseline) & 80.4 \\
            RaLiMP ($k=5$)       & \textbf{80.8} \\
            RaLiMP ($k=10$)      & 80.7 \\
            RaLiMP ($k=22$)      & 80.5 \\
            \bottomrule
        \end{tabularx}
    \end{minipage}
\end{table*}
\begin{table}[h]
\centering
\caption{\small{GPT-2 Large (774M, OpenWebText) validation perplexity (lower is better) at a fixed 30k-iteration budget. RaLiMP improves on the EUGen and SNNK baselines at every $k$.}}
\renewcommand{\arraystretch}{1.25}
\setlength{\tabcolsep}{7pt}
\resizebox{0.7\textwidth}{!}{%
\begin{tabular}{lccccccc}
\toprule
 & \multicolumn{7}{c}{\textbf{FFN layers replaced ($k$)}} \\
\cmidrule(lr){2-8}
Method & 2 & 6 & 11 & 18 & 24 & 27 & 30 \\
\midrule
Low Rank   & 16.83 & 17.57 & 18.79 & 19.06 & 20.91 & 22.66 & 25.31 \\
EUGen  & 16.03 & 16.08 & 16.37 & 16.66 & 17.06 & 17.74 & 18.42 \\
SNNK  & 14.18 & 14.32 & 14.71 & 15.02 & 15.69 & 16.05 & 16.44 \\
RaLiMP (ours) & \textbf{14.10} & \textbf{14.22} & \textbf{14.58} & \textbf{14.90} & \textbf{15.26} & \textbf{15.42} & \textbf{15.69} \\
\midrule
Dense (no replacement) & \multicolumn{7}{c}{14.17} \\
\bottomrule
\label{table:gpt}
\end{tabular}
}
\vspace{-7mm}
\end{table}

\section{Additional Experiments}
In this section, we show that RaLiMP provides further compression over already compressed models. To illustrate this point, we finetune DistilBERT~\cite{sanh2020distilbertdistilledversionbert} which is distilled from BERT~\citep{bert} on the GLUE benchmark~\citep{glue}. RaLiMP beats both SNNK and EUGen and further reduces DistilBERT by \textbf{35}\% while remaining competitive. This shows that our method is complementary to various compression methods and can be combined with them. 

In all our experiments we use 32 random features. RaLimP spends that budget on the polynomial of degree $K=2$ with $8$ random features and the rest of the $24$ features in the rational block is split among $n_{\mathrm{freq}}=3$ frequencies with $m=4$ features each, one centered summand of Eq.~\ref{eq:pf}, and $\tau=4$. We report the MCC score for CoLA, F1 score for MRPC and QQP, Spearman correlation for STSB, and accuracy scores for the other tasks. We use the AdamW optimizer with a weight decay of $10^{-5}$ and a learning rate of $3 \times 10^{-5}$, and a constant scheduler with warmup steps 6\% of the total number of training steps with a batch size of 16.

\begin{figure}
    \centering
    \includegraphics[width=0.85\linewidth]{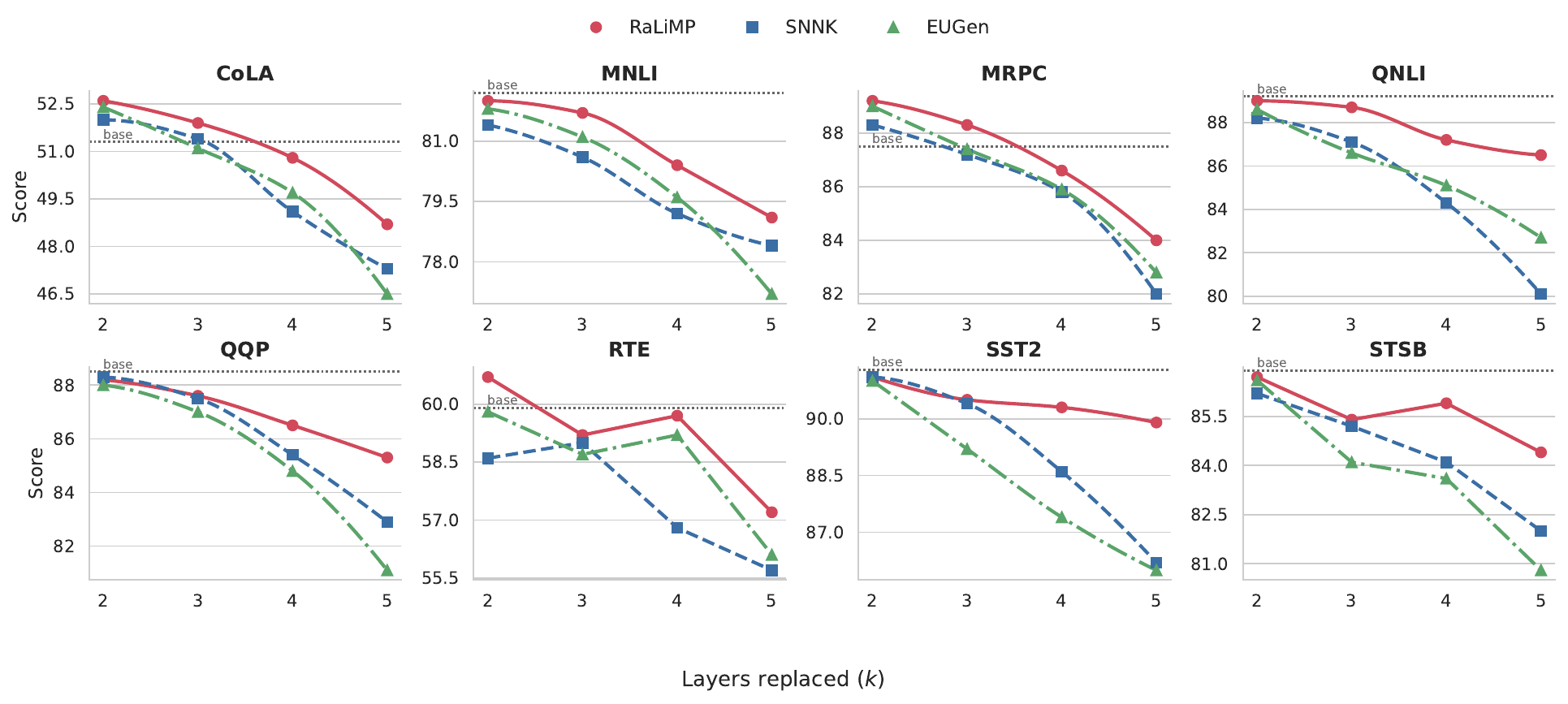}
    \caption{\textbf{GLUE results.} MCC score is reported for CoLA, F1 score is reported for MRPC and QQP, Spearman correlation is reported for STSB, and accuracy scores are reported for the other tasks. RaLiMP is competitive with baseline while being almost 35\% smaller. Results are an average of 5 seeds.}
    \label{fig:glue}
\end{figure}
\section{Additional Related Work}
\label{appendix:related_work}
In this section we present additional related work to complement the works already discussed in Sec.~\ref{sec: intro_related}.

The central theme of our work is to reduce the time and space complexity of neural networks. A substantial body of prior work has focused on reducing these computational bottlenecks via dimensionality reduction using the Johnson-Lindenstrauss transform (JLT)~\citep{dasgupta2003elementary,dasgupta2010sparse,ailon2013almost} or via kernel methods with random features (RF)~\citep{rffs}. Following the influential work of~\citealt{rffs}, considerable effort has been devoted to improving the accuracy of RF-based methods through the use of entangled projections~\citep{yu2016orthogonal,choromanski2017unreasonable}. Kernels that are widely studied are of the form $\textbf{K}_f(\mathbf{x}, \mathbf{y}) := f(\mathbf{x}^{\top}\mathbf{y})$, where $f := \exp$~\citep{performers,likhosherstov2022chefs} or $f$ is a polynomial or power series with positive coefficients~\citep{kar2012random,wacker2022local,wacker2023complex,kacham2023polysketchformer,wacker2024improved}.
However, these approaches are restricted to positive definite kernels, with the recent exception of~\citep{eugens}. Our work extends this linearization to special classes of rational functions $f$.

Kernels have also been used to linearize two-layer neural networks~\citep{cho2009kernel,cho2011analysis}, as well as to analyze more general neural networks through Neural Tangent Kernels (NTKs)~\citep{jacot2018neural} and Path Kernels~\citep{domingos2020every}. While NTK-based methods provide useful insights into training dynamics, particularly in the fine-tuning regime where weights remain close to their pre-trained values~\citep{wei2022more,malladi2023kernel,tomihari2024understanding}, they do not directly reduce the computational cost of MLPs trained from scratch. A different line of work uses kernel approximations to directly linearize the computation of MLP layers. \citealt{snnk} introduced Universal Random Features (URF), which separates the dependence on weights and inputs in certain MLP layers. However, its reliance on Fourier transforms of activation functions limits its direct applicability to unbounded functions such as ReLU. More recently, \citealt{eugens} avoids this requirement by using random features to provide unbiased approximations for MLPs with polynomial activations, while supporting continuous activations through polynomial approximation on bounded intervals.

Other well known methods include pruning~\citep{sun2024simple,ma2023llm}, quantization~\citep{xiao2023smoothquant,liu2025spinquant,zhao2024atom}, knowledge distillation~\citep{li2023prompt,sreenivas2024llm,busbridge2025distillation}, and the use of structured matrices in place of dense ones~\citep{potapczynski2024searching,wei2024building} as common strategies for reducing the computational cost of models with dense layers, including Transformers.
For NeRFs, specialized techniques have been developed to improve efficiency~\citep{tancik2022block,xu2023grid,sun2022direct,reiser2021kilonerf,muller2022instant,hu2022efficientnerf,reiser2023merf}, which does not generalize beyond that domain.

\end{document}